%% file: main.tex
\documentclass[11pt]{article}

\usepackage{setspace}
\renewcommand{\baselinestretch}{1.2}

\newcounter{daggerfootnote}

\usepackage{hyperref}

\makeatletter
\author{Mina Karzand
\thanks{Mina Karzand is with the Wisconsin Institute of Discovery at University of Wisconsin, Madison, WI, USA.
	{\tt\small karzand@wisc.edu}}
	and~Robert D. Nowak
\thanks{Robert  Nowak is with the Department of Electrical and Computer Engineering, University of Wisconsin, Madison, WI, USA.
	{\tt\small rdnowak@wisc.edu}}
\thanks{Parts of this work were presented at the 57th Annual Allerton
  Conference on Communication, Control, and Computing, 2019.}
}

\usepackage[utf8]{inputenc}\usepackage{amsmath}
\usepackage{amsthm}
\usepackage{amsfonts}

\usepackage[utf8]{inputenc} 
\usepackage[T1]{fontenc}    
\usepackage{hyperref}       
\usepackage{url}            
\usepackage{booktabs}       
\usepackage{amsfonts}       
\usepackage{nicefrac}       
\usepackage{microtype}      
\usepackage{color,graphicx}
\input{defNonPActive.tex}

\usepackage{xr}
\title{MaxiMin Active Learning in Overparameterized Model Classes}

\begin{document}

\maketitle
\vspace{-.1in}
\begin{abstract}
  Generating labeled training datasets has become a major bottleneck
  in Machine Learning (ML) pipelines. Active ML aims to address this issue by designing
  learning algorithms that automatically and adaptively select the
  most informative examples for labeling so that human time is not
  wasted labeling irrelevant, redundant, or trivial examples. This
  paper proposes a new approach to active ML with nonparametric or
  overparameterized models such as kernel methods and neural
  networks. In the context of binary classification, the new approach
  is shown to possess a variety of desirable properties that allow
  active learning algorithms to automatically and efficiently identify decision
  boundaries and data clusters.
\end{abstract}

\vspace{-.1in}
\input{introjournal}
\vspace{-.15in}
\input{criteria.tex}

\input{maximinNN.tex}

\input{RKHS.tex}

\input{numsimDBRKHS}

\vspace{-.1in}
\input{NN.tex}
\vspace{-.2in}

\section{Conclusion and Future Work}
The question of designing active learning algorithms in the regime of nonparametric and overparameterized models become more essential as we look at larger models which require bigger training sets. To reduce the human cost of labeling all samples, we can use a pool-based active learning algorithm to avoid labeling non-informative examples. 

Our algorithm does not exploit any assumption about the underlying classifier in selecting the samples
to label. Yet, for a wide range of classifiers, it performs well with provable guarantees. It is designed for the extreme case of the nonparametric setting in which no assumption about the smoothness of the boundary between different classes is made by the learner. 

There are many interesting questions remaining: the behavior of our proposed criterion applied to other classifiers such as kernel SVM instead of minimum norm interpolators, generalization of the criterion to multi-class settings and regression algorithms. The computational complexity of our criterion can also be a serious bottleneck in applications with bigger data-sets and should be addressed in future. Additional numerical simulations, especially with more complex architecture of Neural Networks can also be insightful.

\section*{Acknowledgement}
This work was partially supported by the Air Force Machine Learning Center of Excellence FA9550-18-1-0166.


\bibliographystyle{IEEEtran}
\bibliography{KALrefs}

\appendix
\input{proofNN.tex}
\input{proofs.tex}

\input{Gaussian}
\input{proofsRKHSDB}
\end{document}

%% file: defNonPActive.tex
\usepackage{algorithm}
\usepackage{algorithmic}
\usepackage{amsmath, amsthm, amssymb}
\usepackage{xcolor}
\usepackage{caption}
\usepackage{subcaption}

\def \bx{\boldsymbol{x}}
\def \bw{{\boldsymbol{w}}}
\def \bb{\boldsymbol{b}}
\def \bu{\boldsymbol{u}}
\def\R{{\mathbb R}}

\newcommand{\es}{\ell} 

\newcommand{\spline}{\mathcal{F}}
\newcommand{\data}{\mathtt{D}}

\newcommand{\Uc}{\mathcal{U}}
\newcommand{\Lc}{\mathcal{L}}
\newcommand{\Hc}{\mathcal{H}}
\newcommand{\Fc}{\mathcal{F}}

\newcommand{\Xc}{\mathcal{X}}
\newcommand{\Sc}{\mathcal{S}}
\newcommand{\Bc}{\mathcal{B}}

\newcommand{\xt}{u}
\newcommand{\yb}{\textbf{y}}
\newcommand{\yt}{\widetilde{\textbf{y}}}
\newcommand{\Kb}{\textbf{K}}
\newcommand{\Kt}{\widetilde{\textbf{K}}}

\newcommand{\abv}{\textbf{a}}

\newcommand{\score}{\mathsf{score}}
\newcommand{\argmax}{\mathrm{argmax}}
\newcommand{\argmin}{\mathrm{argmin}}

\newtheorem{corollary}{Corollary}
\newtheorem{theorem}{Theorem}
\newtheorem{lemma}{Lemma}

\newtheorem{prop}{Proposition}

%% file: introjournal.tex
\section{Introduction}

The field of Machine Learning (ML) has advanced considerably in recent
years, but mostly in well-defined domains using huge amounts of
human-labeled training data. Machines can recognize objects in images
and translate text, but they must be trained with more images and text
than a person can see in nearly a lifetime.  The computational
complexity of training has been offset by recent technological
advances, but the cost of training data is measured in terms of the
human effort in labeling data. People are not getting faster nor
cheaper, so generating labeled training datasets has become a major
bottleneck in ML pipelines. Active ML aims to address this issue by
designing learning algorithms that automatically and adaptively select
the most informative examples for labeling so that human time is not
wasted labeling irrelevant, redundant, or trivial examples. This paper
explores active ML with nonparametric or overparameterized models such
as kernel methods and neural networks.

Deep neural networks (DNNs) have revolutionized machine learning
applications, and theoreticians have struggled to explain their
surpising properties.  DNNs are highly overparameterized and often fit
perfectly to data, yet remarkably the learned models generalize well
to new data.  A mathematical understanding of this phenomenom is
beginning to emerge
\cite{ma2018power,belkin2018understand,vaswani2018fast,belkin2018does,belkin2018reconciling,arora2019fine,belkin2018overfitting,hastie2019surprises}.
This work suggests that among all the networks that could be fit to the
training data, the learning algorithms used in fitting favor networks
with smaller weights, providing a sort of implicit regularization.
With this in mind, researchers have shown that shallow (but wide)
networks and classical kernel methods fit to the data but regularized
to have small weights (e.g., minimum norm fit to data) can generalize well
\cite{belkin2018understand,belkin2019two,hastie2019surprises,liang2018just}.

Despite the recent success and new understanding of these systems, it
still is a fact that learning good neural network models can
require an enormous number of labeled data.  The cost of obtaining
labels can be prohibitive in many applications.
This has prompted researchers to investigate active ML 
for kernel methods and neural networks
\cite{tong2001support,cohn1994improving,sener2017active,gal2017deep,wang2017cost,shen2017deep}.
None of this work, however, directly addresses overparameterized and
interpolating regime, which is the focus in this paper.
Active ML algorithms have access to a large but unlabeled
dataset of examples and sequentially select the most ``informative'' examples for
labeling 
\cite{settles2009active,settles2012active} .  This can reduce the
total number of labeled examples needed to learn an accurate model.

Broadly speaking, active ML algorithms adaptively select
examples for labeling based on two general strategies
\cite{dasgupta2011two}. The first is to select examples that rule-out
as many (incompatible) classifiers as possible at each step. In
effect, this leads to algorithms that tend to label examples near
decision boundaries. The second strategy involves discovering cluster
structure in unlabeled data and labeling representative examples from
each cluster.  We show that our new MaxiMin active learning approach
automatically exploits both these strategies, as depicted in Figure~\ref{2face}.

\begin{figure}[h]
\centerline{\includegraphics[width=.8\linewidth]{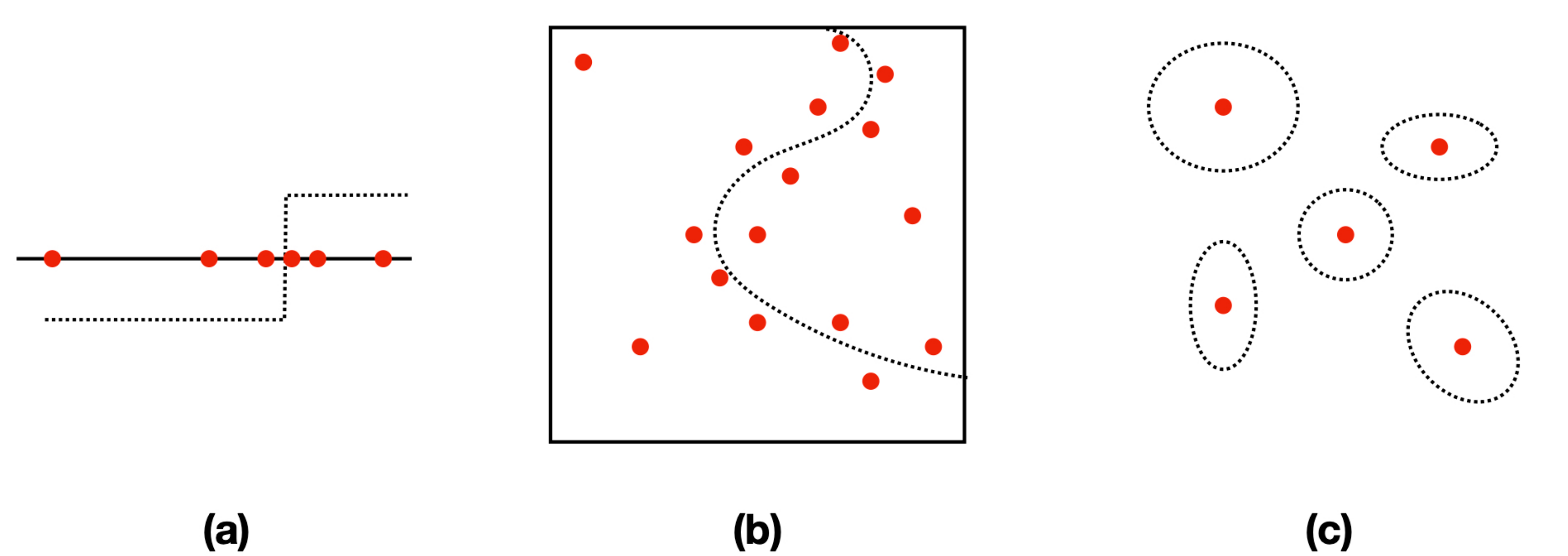}}
\caption{MaxiMin Active Learning strategically selects examples for labeling (red points). (a) reduces to binary search in simple 1-d threshold
  problem setting; (b) labeling is focused near decision boundary in
  multidimensional setting; (c) automatically discovers clusters and
  labels representative examples from each.}
\label{2face}
\end{figure}

This paper builds on a new framework for active learning in the
overparameterized and interpolationg regime, focusing on kernel
methods and two-layer neural networks in the binary classification
setting.  The approach, called {\em MaxiMin Active Learning}, is based
on mininum norm interpolating models.  Roughly speaking, at each
step of the learning process the maximin criterion requests a label
for the example that is most difficult to interpolate. A minimum norm
interpolating model is constructed for each possible example and the
one yielding the largest norm indicates which example to label next.
The rationale for the maximin criterion is that labeling the most
challenging examples first may eliminate the need to label many of the
other examples.

The maximin selection criterion is studied through experiments and mathematical
analysis.  We prove that the criterion has a number of desirable
properties:
\begin{enumerate}
\item[$\bullet$] It tends to label examples near the current
  (estimated) decision boundary and close to oppositely
  labeled examples,  allowing the active learning algorithm to focus
on learning decision boundaries.
\item[$\bullet$] It reduces to
optimal bisection in the one-dimensional linear classifier
setting.  
\item[$\bullet$] A data-based form of the criterion also provably
discovers clusters and also automatically generates labeled coverings
of the dataset.
\end{enumerate}
Experimentally, we show that these properties generalize in
several ways.  For example, we find that in multiple dimensions the
maximin criterion leads to a multidimensional bisection-like process
that automatically finds a portion of the decision boundary and then
locally explores to efficiently identify the complete boundary.  We
also show that MaxiMin Active Learning can learn hand-written
digit classifiers with far fewer labeled examples than traditional
passive learning based on labeling a randomly selected subset of examples.

%% file: criteria.tex
\section{A New Active Learning Criterion}
\label{s:DefScore}

At each iteration of the active learning algorithm,
looking at the currently labeled set of samples, a new unlabeled point
is selected to be labeled.  The criterion we are proposing to pick the
samples to be labeled is based on a `maximin' operator.  We will
describe the criterion in its most general form along with the
intuition behind this choice of criterion.  In the remainder of the
paper, we will go through some theoretical results about the
properties of variations of this criterion in various setups along
with some additional descriptive numerical evaluations and
simulations.

\subsection{Nonparametric Pool-based Active Learning}

At each time step, the algorithm has access to a pool of labeled
samples and a set of unlabaled samples.  In other words, we have a
partially labeled training set.  Let
$\Lc=\{(x_1,y_1),\cdots,(x_L,y_L)\}$ be the set of labeled examples so
far.  We assume $x_i\in\mathcal{X}$ where $\mathcal{X}$ is the input/feature space and binary valued labels $y_i\in\{-1,+1\}$.  Let $\Uc\subseteq \mathcal{X}$ be
the set of unlabeled samples.

{\em In the interpolating regime, the goal is to correctly label all
the points in $\Uc$ so that the training error is zero.} Passive learning generally
requires labeling every point in $\Uc$.  Active learning sequentially selects
points in $\Uc$ for labeling with the aim of learning a correct
classifier without necessarily labeling all of $\Uc$.  Our setting can
be viewed as an instance of {\em pool-based} active learning. 

At each iteration, one unlabeled sample, $u^*\in\mathcal{U}$ is
selected, labeled and added to the pool of labeled samples.
The selection process is designed to pick the samples which are most
\textit{informative} upon being labeled.  The proposed notion of score
is the measure of informativeness of each sample $u\in\Uc$ at each
time: the score of each unlabaled sample is computed, and the sample
with the largest score is selected to be labeled.
\begin{align}
\label{eq:argmaxScore}
u^* = \argmax_{u\in\mathcal{U}} \score(u)\,.
\end{align}
If there are multiple maximizers, then one is selected uniformly at
random.  Note that for any unlabeled sample $u\in\Uc$, the value of
$\score(u)$ depends implicitly on the set of currently labeled points,
$\Lc$. That is, information gained by labeling $u$ depends on the
current knowledge of the learner.  To define our proposed notion of
$\score$, we define minimum norm interpolating function and introduce
some notations next.

\subsection{Minimum norm interpolating function}
Let $\Fc$ be a class of functions mapping $\mathcal{X}$ to
$\mathbb{R}$,  where $\mathcal{X}$ is the input/feature space,.  We
assume the class $\Fc$ is rich enough to interpolate the training data.  For
example, $\Fc$ could be a nonparametric infinite dimensional
Reproducing Kernel Hilbert Space (RKHS) or an overparameterized neural
network representation.

Given the set of labeled samples, $\Lc$, and a class of functions
$\Fc$, let $f\in\Fc$ be the interpolating function such that
$f(x_i)=y_i$ for all $(x_i,y_i)\in\Lc$.  Note that there may be
many functions that interpolate a discrete set of points such as
$\Lc$. Among these, we choose $f$ to be the minimum norm interpolator:
\begin{align}
\label{eq:defF}
f(x)
:=&\,  \argmin_{g \in \mathcal{F}}\,\, \|g\|_{\mathcal{F}} 
\\
\text{s.t. } & \quad 
 g(x_i)  =  y_i, 
 \text{ for all } (x_i,y_i)\in\Lc\,.
 \notag
\end{align}
Clearly, the definition of $f$ depends on the set of currently labeled
samples $\Lc$ and the function norm $\|\cdot\|_{\mathcal F}$, although we omit
these dependencies for ease of notation. The choice of $\mathcal{F}$ and
the norm $\|\cdot\|_{\mathcal{F}}$ is application dependent.  In this
paper, we focus on (1) function classes represented by an
overparameterized neural network representation with the $\ell^2$ norm
of the weight vectors and (2) reproducing kernel Hilbert spaces with the
corresponding Hilbert norm.

For unlabeled points $u\in\Uc$ and $\es\in\{-1,+1\}$, define $f^{\xt}_{\es} (x)$ is the
minimum norm interpolating function based on current set of labeled
samples $\Lc$ and the point $\xt\in \Uc$ with label $\es$:
\begin{align}
\label{eq:defFtu}
f^u_{\es}(x)
 :=&\, \argmin_{g\in \mathcal{F} }\,\, \|g\|_{\mathcal{F}}  \\
\text{s.t. } & \quad 
 g(x_i)  =  y_i, 
 \text{ for all } (x_i,y_i)\in\Lc
  \notag
 \\
 & \quad 
  g(u)  =  \es\,. 
   \notag
\end{align}
We use this definition in the next subsection to define the notion of $\score$.

\subsection{Definition of proposed notion of $\score$}
Roughly speaking, we want our selection criterion to prioritize
labeling the most ``informative'' examples.  Since the ultimate goal
is to correctly label every example in $\Uc$, we design $\score(u)$ to measure
the how hard it is to interpolate after adding $u$ to the set of
labeled points. 
The intuition is that attacking the most challenging points in the input
space first may eliminate the need to label other `easier' examples
later. 

Note that we need to compute $\score(u)$ without knowing the label of
$u$. To do so, we come up with an estimate of label of $u$, denoted by
$\es(u)\in\{-1,+1\}$ and compute $\score(u)$ assuming that upon
labeling, $u$ will be labeled $\es(u)$.  We propose the following
criterion for choosing $\es(u)$:
\begin{align}
\es(u) 
& := 
\displaystyle{\argmin_{\es\in\{-1,+1\}} } \|f_{\es}^u(x)\|_{\mathcal{F}} \,.
\label{eq:t1def}
\end{align}
 Operating in the interpolating regime,
we estimate the label of any unlabeled sample, $u$, to be the one   that yields the minimum norm interpolant (i.e.,
the ``smoother'' of the two interpolants among the two possible functions $f_+^u(x)$ and $f_-^u(x)$).

Define 
\begin{equation}
f^u(x):=f_{\es(u)}^u(x)
\label{eq:defnextinterp}
\end{equation}
to be the interpolating function after adding the sample $u$ with the label $\es(u)$, defined in~\eqref{eq:t1def}. 

We propose two notions of $\score$. For $\xt \in \Uc$, define
\begin{eqnarray}
\score_{\Fc}(\xt)  & = & \|f^{\xt} (x)\|_{\Fc}\, \label{fscore} \\
\score_{\data}(\xt)  & = & \|f^{\xt} (x) - f(x)\|_{\data} \, \label{dscore}
\end{eqnarray}
where $\|\cdot\|_{\mathcal F}$ is the norm associated the the function
space $\mathcal{F}$. The function  $f$ is the minimum norm interpolator of the labeled examples in
$\Lc$ (defined in~\eqref{eq:defF}), and $f^{\xt} (x)$ is defined \eqref{eq:defnextinterp} as the minimum norm interpolator after adding $u$ with the estimated label $\es(u)$ to the set of labeled points. 
Also, define
\begin{equation}
\label{eq:DBNorm}
\|g\|_{\data} \ = \ \int_{\mathcal X} |g(x)|^2 \, \mbox{d}P_X(x) \
, 
\end{equation}
where $P_X$ is
the distribution of $x$. In practice, $P_X$ is the empirical
distribution of $\Uc$. We refer to the (\ref{fscore}) as the {\em function norm score}
and (\ref{dscore}) as the {\em data-based norm score}\footnote{Operationally, to compute the data-based norm of any function, the algorithm uses the probability mass function of set of unlabeled points as a proxy for the input probability density function over the feature space $\mathcal{X}$. In particular, the algorithm approximates $\|g\|_{\data} $ by the average of the function over the set of unlabeled points:
$\|g\|_{\data} \approx \frac{1}{|\mathcal{U}|}\sum_{u\in\mathcal{U}} |g(u)|^2$. High density of set of unlabeled points and some mild regularity conditions guarantee that this is a good approximation. Throughout the paper, we use~\eqref{eq:DBNorm} to prove theoretical statements and its approximation in the numerical simulations.}.

The distinction between the two definitions of the $\score$ function
 is as follows.  Scoring unlabeled points
according to the definition $\score_{\Fc}$ priotorizes labeling the
examples which result in minimum norm interpolating functions with
largest norm. Since the norm of the function can be associated with
its smoothness, roughly speaking, this means that this criterion picks
the points which give the least smooth interpolating functions.
However, $\score_{\Fc}$ is insensitive to the distribution of
data. The data-based $\score_{\data}$, in contrast, is sensitive to the
distribution of the data.  Measuring the difference between the new
interpolation $f^u$ and the previous one makes this also sensitive to
the structure of the function class.

With these definitions in place, we state the MaxiMin Active Learning
criterion as follows.  Given labeled data $\Lc$, the next example $u^*
\in \mathcal{U}$ to
label is selected according to
\begin{eqnarray*}
f^u & = & \arg\min_{f \in \{f_+^u, f_-^u\}} \|f\|_{\mathcal F}  \ , \
          \forall u\in \mathcal{U} \\
u^*& = & \arg\max_{u\in \mathcal{U}} \score(u)
\end{eqnarray*}
with either $\score_{\Fc}$ or $\score_{\data}$.

%% file: maximinNN.tex
\section{MaxiMin Active Learning with Neural Networks}

\subsection{Overparameterized Neural Networks and Interpolation} 
Neural networks are often highly
overparameterized and exactly fit to training data, yet remarkably the
learned models generalize well to new data.  A mathematical
understanding of this phenomenom is beginning to emerge
\cite{ma2018power,belkin2018understand,vaswani2018fast,belkin2018does,belkin2018reconciling,arora2019fine,belkin2018overfitting,hastie2019surprises}.
This work suggests that among all the networks that could be fit to the
training data, the learning algorithms used in training favor networks with smaller
weights, providing a sort of implicit regularization.  With this in
mind, researchers have shown that even shallow networks and
classical
kernel methods fit to the data but regularized to have small weights (e.g.,
minimum norm fit to data) can generalize well
\cite{belkin2018understand,belkin2019two,hastie2019surprises,liang2018just}.
The functional mappings generated by wide, two-layer neural networks
with Rectified Linear Unit (ReLU) activation functions were studied
in~\cite{nn-linear-spline}.  It is shown that exactly fitting such networks to
training data subject to  minimizing the $\ell^2$-norm of the network
weights results in a linear spline interpolation. This result was extended to a
broad class of interpolating splines by appropriate choices of
activation functions \cite{parhi2019minimum}.  Our analysis of the
MaxiMin active learning with neural networks will leverage these connections.

\subsection{Neural Network Regularization}
It has been long understood that the size of neural network weights,
rather than simply the number of weights/neurons, characterizes the
complexity of neural networks \cite{bartlett1998sample}.  Here we
focus on two-layer neural networks with ReLU activation functions in
the hidden layer.  If $\bx \in \R^d$ is input to the network, then the
output is computed by the function
\begin{eqnarray}
f_{\bw,\bb,c}(\bx) & = & \sum_{n=1}^N v_n \, \sigma(\bu_n^T\bx + b_n) \ + \ c \ , 
\label{twolayerNN}
\end{eqnarray}
where $\sigma(\cdot) = \max\{0,\cdot\}$ is the ReLU activation,
$\bw := \{v_n,\bu_n\}_{n=1}^N$ are the ``weights'' of the network, and
$\bb:=\{b_n\}$ and $c$ are constant ``bias'' terms.  The ``norm'' of
$f_{\bw,\bb,c}$ is defined as $\|f_{\bw,\bb,c}\| := \|\bw\|_2$, the
$\ell_2$-norm of the vector of network weights.  We use the term norm
in quotes because technically the weight norm does not correspond to a
true norm on the function $f_{\bw,\bb,c}$ since, for example, constant
functions $f_{\bw,\bb,c} = c$ have $\|\bw\|_2=0$.
From now on we will drop the subscripts and just
write $f$ for
ease of notation.  Let $\{(x_i,y_i)\}_{i=1}^M$ be a set of
training data.  The minimum ``norm'' neural network interpolation of
these data is the solution to the optimization
$$\min_{\bw} \|\bw\|_2 \mbox{ subject to } f(x_i)=y_i, \
i=1,\dots,M. $$
A solution exists if the number of neurons $N$ is sufficiently large
(see Theorem 5.1 in \cite{pinkus1999approximation}).

In Section~\ref{s:numsim} we explore the behavior of MaxiMin active
learning through numerical experiments using both the function
``norm'' score and the data-based norm score. In all our
experiments and theory, we assume the binary classification setting
where $y_i = \pm 1$.  Broadly
speaking, we observe the following behaviors.
\begin{enumerate}
\item[$\bullet$] With the function ``norm'' score the MaxiMin active
  learning algorithm tends to sample aggressively in the vicinity of
  the boundary, prefering to gather new labels between the closest
  oppositely labeled examples. 
\item[$\bullet$] The data-based norm score is sensitive to the distribution
  of the data. It strikes a balance between exploiting regions between
  oppositely labeled examples (as in the function-based case) and
  exploring regions further away from labeled examples.  Thus we see
  evidence that the data-based norm can effectively seek out the
  decision boundary and explore data clusters.
\end{enumerate}
These behaviors are supported by a formal analysis of MaxiMin active
learning in one dimension, discussed next.

\subsection{MaxiMin Active Learning in One-Dimension}
Our analysis of MaxiMin active learning with neural
networks will focus on the behavior in one-dimension.  We show that
MaxiMin active learning with a two-layer ReLU netwok recovers optimal
bisection learning strategies.
The following characterization of minimum ``norm''
neural network interpolation in one-dimension follows from
\cite{nn-linear-spline,parhi2019minimum} (see Theorem 4.4 and
Proposition 6.1 in \cite{parhi2019minimum}).
\begin{theorem}
\label{t:minnormNN}
Let $f:\R\rightarrow\R$ be a two-layer neural network with ReLU
activation functions and $N$ hidden nodes as in~\eqref{twolayerNN}. Let $\{(x_i,y_i)\}_{i=1}^M$ be a set of
training data.  If $N\geq M$, then a solution to the optimization
$$\min_{\bw} \|\bw\|_2 \mbox{ subject to } f(x_i)=y_i, \
i=1,\dots,M$$
is a minimal knot linear spline interpolation of the points
$\{(x_i,y_i)\}_{i=1}^M$.
\end{theorem}
In our analysis, we exploit the equivalence between minimum ``norm''
neural networks and linear splines. Specifically, a solution to the
optimization is an interpolating function that is linear between each 
pair of neighboring points.  This ensures that given a pair of
neighboring labeled points $x_1$ and $x_2$ and any unlabeled point
$x_1<u<x_2$, adding $u$ to the set of labeled points can only
potentially change the interpolating function between $x_1$ and $x_2$.
To eliminate uncertainty in the boundary conditions of the
interpolation, we assume that the neural network is initialized by labeling the
leftmost and rightmost points in the dataset and forced to have a
constant extension to the left and right of these points (this can be
accomplished by adding two artificial points to the left and right
with the same labels as the true endpoints).

The main message of our analysis is that MaxiMin active learning with
two-layer ReLU networks recovers optimal bisection (binary search) in
one-dimension.  This is summarized by the next corollary which
follows in a straightforward fashion from
Theorems~\ref{thm:BinarySearchLinearSpline} and
\ref{thm:DataSearchLinearSpline}.

\begin{corollary}
  Consider $N$ points uniformly distributed in the interval $[0,1]$
  labeled according to a $k$-piecewise constant function $f$ so that
  $y_i=f(x_i)\in\{-1,+1\}$, $i=1,\dots,N$, and length of the pieces
  are $\Theta(1/K)$.  Then after labeling $O(k\log N)$ examples, the
  MaxiMin active learning with a two-layer ReLU network correctly
  labels all $N$ examples (i.e., the training error is zero).
\label{cor:binsearchNN}
\end{corollary}

\noindent The corollary follows from the fact that the MaxiMin criteria (both
function norm and data-based norm) selects the next example to label
at the midpoint between neighboring and oppositely labeled examples
(i.e., at a bisection point).  This is characterized in the next two
theorems.  First we consider the function ``norm'' criterion.
The proof of the following theorem appears in Appendix~\ref{s:proofNN}.

\begin{theorem}\label{thm:BinarySearchLinearSpline}
Let $\Lc$ be a set of labeled examples and let $u$ be an unlabeled
example. Let $f_{+}^{u}$ be the minimum ``norm'' interpolator of
$\Lc\cup (u,+1)$ and let $f_{-}^{u}$ be the minimum ``norm'' interpolator of
$\Lc\cup (u,-1)$.  
Define the score of an unlabeled example $u$ as
 $\score_{\Fc}(u) \ = \ \min\{\|f_{+}^{u} \|, \|f_{-}^{u}\|\}$,
where $\|f\| = \|\bw\|_2$, the neural network weight norm.
 Then, the selection criterion based on $\score_{\Fc}$ has the following properties
\begin{enumerate}
\item Let $x_1$ and $x_2$ be two oppositely labeled neighboring points
  in $\Lc$, \textit{i.e.},  no other points between $x_1$ and $x_2$
  have been labeled and $y_1\neq y_2$. Then for all $x_1<\xt<x_2$,
$\score_{\Fc} \left(\frac{x_1 + x_2}{2}\right) 
\geq
 \score_{\Fc}(\xt)$.
\item Let $x_1<x_2$ and $x_3<x_4$ be two pairs of oppositely labeled neighboring points 
(\textit{i.e.},  $y_1\neq y_2$ and $y_3\neq y_4$) 
such that $x_2-x_1\geq x_4-x_3$. Then,
$$\score_{\Fc} \left(\frac{x_1 + x_2}{2}\right)
\  \leq \
 \score_{\Fc} \left(\frac{x_3+ x_4}{2}\right)\,.$$
 \item 
  Let $x_5$ and $x_6$ be two identically labeled neighboring points in $\Lc$, \textit{i.e.}, $y_5= y_6$. Then for all $x_5<\xt<x_6$, the function 
$\score_{\Fc}(\xt)$ is constant.
\item For any pair of neighboring oppositely labeled points $x_1$ and $x_2$, 
any pair of neighboring identically labeled points $x_5$ and $x_6$,
 any $x_1<u<x_2$ and any $x_5<v<x_6$, we have
 \[\score_{\Fc} (v)\leq
 \score_{\Fc} (u) \,.\]
\end{enumerate}
\end{theorem}

Now we turn to the data-based norm.  Here we observe the effect of the
data distribution on the bisection properties.  The properties mirror
those in Theorem~\ref{thm:BinarySearchLinearSpline} except in the
case of the second property.  The data-based norm criterion tends to
sample in the {\em largest} (most data-massive) interval between
oppositely labeled points, whereas the function-based norm criterion
favors points in the {\em smallest} interval.

\begin{theorem}\label{thm:DataSearchLinearSpline}
Let the distribution $\mathbb{P}(X)$ be uniform over an interval. 
Let $\Lc$ be a set of labeled examples and let $u$ be an unlabeled
example. Let $f_{+}^{u}$ be the minimum ``norm'' interpolator of
$\Lc\cup (u,+1)$ and let $f_{-}^{u}$ be the minimum ``norm'' interpolator of
$\Lc\cup (u,-1)$ and let $f^u = \arg_{g \in \{f_+^u,f_-^u\}} \|g\|$ consistent with notations in~\eqref{eq:defFtu} and~\eqref{eq:defnextinterp}.
Then
$\score_{\data}(u)= \int |f^u(x)-f(x)|^2\, dP_X(x)$, where $f$ is
the minimum ``norm'' interpolator based on the labeled data $\Lc$.
 Then, the selection criterion based on $\score_{\data}$ has the following properties.
\begin{enumerate}
\item Let $x_1$ and $x_2$ be two oppositely labeled neighboring points in $\Lc$, \textit{i.e.}, $y_1\neq y_2$. Then for all $x_1<\xt<x_2$
$\score_{\data} \left(\frac{x_1 + x_2}{2}\right) 
\geq
 \score_{\data}(\xt)$.
\item Let $x_1<x_2$ and $x_3<x_4$ be two pairs of oppositely labeled neighboring labeled points 
(\textit{i.e.},  $y_1\neq y_2$ and $y_3\neq y_4$) 
such that $x_2-x_1\geq x_4-x_3$.  If the unlabeled points are uniformly
distributed in each interval and the number of points is in
$(x_1,x_2)$ is less than the number in $(x_4,x_3)$, then
$$\score_{\data} \left(\frac{x_1 + x_2}{2}\right)
 \geq 
 \score_{\data} \left(\frac{x_3+ x_4}{2}\right)\,.$$
 \item 
  Let $x_5$ and $x_6$ be two identically labeled neighboring points in $\Lc$, \textit{i.e.}, $y_5= y_6$. Then for all $x_5<v<x_6$, we have
$\score_{\data}(v)=0$.
\item For any pair of neighboring oppositely labeled points $x_1$ and $x_2$, 
any pair of neighboring identically labeled points $x_5$ and $x_6$,
 any $x_1<u<x_2$ and any $x_5<v<x_6$, we have
 \[\score_{\data} (v)\leq
 \score_{\data} (u) \,.\]
\end{enumerate} 
 
\end{theorem}

The proof appears in Appendix~\ref{s:proofNNDB}.

%% file: RKHS.tex
\section{Interpolating Active Learners in an RKHS}
\label{sec:properties}
In this section, we will focus on minimum norm interpolating functions
in a Reproducing Kernel Hilbert Space (RKHS).  We present theoretical properties for general RKHS
settings, detailed analytical results in the one-dimensional setting,
and numerical studies in multiple dimensions. Broadly speaking, we establish the following
properties: the proposed score functions \\ \vspace{-.1in} \\
\noindent $\bullet$ tend to select examples near the decision boundary of
  $f$, the current interpolator; \\ \vspace{-.1in} \\
\noindent $\bullet$ the score is largest for unlabeled examples near the decision
  boundary {\em and} close to oppositely labeled examples, in effect
  searching for the boundary in the most likely region of the
  input space;\\ \vspace{-.1in} \\
\noindent $\bullet$  in one dimension the interpolating active learner
  coincides with an optimal binary search procedure; \\ \vspace{-.1in} \\
\noindent $\bullet$ using data-based function norms, rather than the RKHS
  norm, the interpolating active learner executes a tradeoff between
  sampling near the current decision boundary and sampling in regions
  far away from currently labeled examples, thus exploiting cluster
  structure in the data.
\vspace{-.05in}


\subsection{Kernel Methods}
\label{sec:RKHS}

A Hilbert space $\Hc$ is associated with an inner product:
$\langle f,g\rangle_{\Hc}$ for $f,g\in\Hc$. This induces a norm
defined by $\|f\|_{\Hc} = \sqrt{\langle f,f\rangle}_{\Hc}$.
A symmetric bivariate function
$K: \mathcal{X}\times \mathcal{X} \to \mathbb{R}$ is positive
semidefinite if for all $n\geq 1$, and points $\{x_i\}_{i=1}^n$, the
matrix $\Kb$ with element $\Kb_{i,j}=K(x_i,x_j)$ is positive
semidefinite (PSD). These functions are called PSD kernel functions.
 A PSD kernel constructs a Hilbert space, $\Hc$ of functions on
$f:\mathcal{X}\to \mathbb{R}$. For any $x\in\mathcal{X}$ and any
$f\in\Hc$, the function $K(\cdot,x)\in \Hc$ and
$\langle f,K(\cdot,x)\rangle_{\Hc} = f(x)$.
Throughout this section, we assume $K(x,x)=1$. 
 

For the set of labeled samples $\Lc=\{(x_1,y_1),\cdots,(x_L,y_L)\}$ with $y_i\in \{-1,+1\}$, let the function $f(x)$
 be decomposed as
\begin{align}
f(x) & = \sum_{i=1}^L \alpha_i k(x_i,x)
\label{eq:InterpFnc}
\\
\text{with} \quad
\alpha & = \Kb^{-1} \yb, \nonumber
\end{align}
where $\Kb=\big[\Kb_{i,j}\big]_{i,j}$ is the $L$ by $L$ matrix such that $\Kb_{i,j}=k(x_i,x_j)$ and $\yb=[y_1,\cdots,y_L]^T$. 
Using reproducible kernels  implies that $f(x)\in\Hc$ for the a RKHS $\Hc$. Then, $f(x)$ defined above is the minimum Hilbert norm interpolating function defined in~\eqref{eq:defF}. 
Using  the property $\langle K(x_i,\cdot) , K(x_j,\cdot) \rangle =
K(x_i,x_j)$, we have 
	\[\|f(x)\|_{\Hc}^2 
	= \alpha^T\, \Kb\, \alpha
	=
	\yb^T\, \Kb^{-1}\, \yb \,.\]

For $u\in\Uc$ and ${\es}\in\{-1,+1\}$, the minimum norm interpolating unction $f^u_{\es}(x)$, defined in~\eqref{eq:defFtu} (based on currently labeled samples $\Lc$ and sample $u$ with label ${\es}$) is derived similarly :
\begin{align}
f^u_{\es}(x) & = \sum_{i=1}^L \widetilde{\alpha}_i k(x_i,x) + \widetilde{\alpha}_{L+1}  k(u,x)
\label{eq:InterpFnctu}
\\
\text{with} \quad
\widetilde{\alpha} & = \Kt_u^{-1} \yt_{\es}, \nonumber
\end{align}
where
\begin{align}
\label{eq:defKtYt}
\Kt_u = 
\begin{bmatrix}
\Kb& \abv_u\\
\abv_u^T & b
 \end{bmatrix}
 \,,
 \,\,
\abv_u=\begin{bmatrix}
k(x_1,u)
\\
\vdots
\\k(x_L,u)
\end{bmatrix}
\,,
\,\,
\yt_{\es}=
\begin{bmatrix}
\yb
\\
{\es}
\end{bmatrix}\,,
\,\,
\text{and}
\,\,
b=K(u,u)\,.
\end{align}
Throughout this paper, we use kernel such that $K(x,x)=1$ for all $x\in\mathcal{X}$.



\subsection{Properties of General Kernels for Active Learning }
\label{s:PropRKHS}
We first show that using kernel based function spaces for interpolation,  
${\es}(u)$ defined in~\eqref{eq:t1def} coincides with the sign of value of current interpolator at $u$.

\begin{prop}
\label{p:labelKernel}
	For $\xt\in \mathcal{X}$ and ${\es}\in\{-,+\}$, 
	define $f(x)$ and  $f_{\es}^{\xt}(x)$ according to~\eqref{eq:defF} and~\eqref{eq:defFtu} in Section~\ref{s:DefScore}.
	Then, $\es(u)$ defined in~\eqref{eq:t1def} satisfies 
	\[{\es}(u)=
 \begin{cases}
 +1 \quad \quad \text{ if } f(u)\geq 0
\\
 -1 \quad \quad \text{ if } f(u)<0\,.
 \end{cases} \]
\end{prop}

\begin{proof}
	
Let $\yt_{\es}=[y_1,\cdots,y_n, {\es}]^T$, 
 $\abv_u=[K(x_1,\xt),\cdots,K(x_n,\xt)]^T$
  and $b = K(\xt,\xt)=1$. 
  Let $\Kb$  be the kernel matrix for the elements in $\Lc$ and $\widetilde{\Kb}_u$ be the kernel matrix for the elements in $\Lc\cup \{u\}$, as defined in~\eqref{eq:defKtYt}. 
  Then,    for ${\es}\in\{-1,+1\}$
\begin{align*}
\|f^{\xt}_{\es} (x)\|_{\Hc}^{2} 
& = \yt_{\es}^T\,\, \widetilde{\Kb}_u^{-1} \,\, \yt_{\es}
 \overset{(a)}{=} 
 \yb^T\,\, 
 \Big( \Kb- \abv_u \, \abv_u^T\Big)^{-1}\,\, \yb - 2{\es}\, \,\yb^T \, \Big( \Kb- \abv_u \, \abv_u^T \Big)^{-1}
\abv
  + \Big( 1-\abv_u^T  \, \Kb^{-1} \abv_u \Big)^{-1}
  \\
  &
   \overset{(b)}{=} 
  \yb^T\,\, \Kb^{-1}\,\, \yb +
   \frac{\Big(1- {\es}\, \yb^T\,\, \Kb^{-1} \abv_u\Big)^2 }{ 1-\abv_u^T  \, \Kb^{-1} \abv_u}
   \overset{(c)}{=} 
 \|f(x)\|_{\Hc}^{2}+ \frac{\big[1-{\es}\, f(u)\big]^2}{ 1-\abv_u^T  \, \Kb^{-1} \abv_u}\,.
 \end{align*}
where  Schur's complement formula gives (a)  and Woodbury Identity
with some algebra algebra gives (b).  We are using the property that $K(x,x)=1$ and the diagonal elements of matrix $\widetilde{\Kb}_u$ are equal to one. 
(c) uses~\eqref{eq:InterpFnc} for the minimum norm interpolating function based on $\Lc$, \textit{i.e.}, $f(x)$.
 Hence, 
 	$\|f^{\xt}_{+} (x)\|_{\Hc} > \|f^{\xt}_{-} (x)\|_{\Hc}$ if and only if
 	$f(u) < 0$ which gives the statement of proposition.
\end{proof}


\vspace{-.05in}

\subsection{Radial Basis Kernels}
From here on, we will focus on  minimum norm interpolating functions with radial basis kernels. 
The kernel functions we use have the following form: For $x,x'\in\mathbb{R}^d$, $h>0$ and $p>1$, let
\begin{align}
\label{eq:kernelDef}
k_{h,p}(x,x') = \exp\Big(-\frac{1}{h} \| x-x'\|_p\Big)\,,
\end{align}
where $\|x\|_p:=\big(\sum_{i=1}^d x_i^p\big)^{1/p} $ is the $\ell_p$ norm and  $ \| x-x'\|_p$ is the Minkowski distance satisfying the triangle inequality. For $p=1,2$ this category of kernels  construct Reproducing Kernel Hilbert Spaces.
When the parameters $h$ and $p$ are specified, we denote the kernel function $k_{h,p}(x,y)$ by $k(x,y)$.

\subsection{Laplace Kernel in One Dimension}
\label{s:LaplaceKernelOneD}
To develop some intuition, we consider active learning in 
one-dimension.  The sort of target function we have in mind is a
multiple threshold classifier.  Optimal active learning in this
setting coincides with binary search.  We now show that the proposed selection criterion based on 
$\score_{\Hc}$ with Hilbert norm associated with the Laplace kernels result in an optimal active
learning in one dimension (proof in Appendix~\ref{s:AppLaplace}).
\begin{prop}
\label{p:Laplace-1D}
[Maximin criteria in one dimension with Laplace kernel]\label{prop:BinarySearch}
Define $K(x,x')=\exp(-|x-x'|/h)$ to be the Laplace kernel in one dimension and the minimum norm interpolator function defined in Section~\ref{sec:RKHS}. 
Let the selection criterion be based on $\score_{\Hc}(u)$ function defined in~\eqref{fscore} with the Laplace kernel Hilbert norm. 
 Then the following statements hold for any value of $h>0$:
\begin{enumerate}
\item Let $x_1$ and $x_2$ be two neighboring labeled points in $\Lc$. Then 
$\score_{\Hc}\left(\frac{x_1 + x_2}{2}\right) \geq \score_{\Hc}(\xt)$ for all $x_1<\xt<x_2$.
\item Let $x_1<x_2$ and $x_3<x_4$ be two pairs of neighboring  labeled points such that $x_2-x_1\geq x_4-x_3$, then
\begin{itemize}
\item if  $y_1\neq y_2$ and $ y_3 = y_4$. Then
$\score_{\Hc}\left(\frac{x_1 + x_2}{2}\right) \geq \score_{\Hc}\left(\frac{x_3+ x_4}{2}\right)$.
\item if  $y_1= y_2$ and $ y_3 \neq y_4$. Then
$\score_{\Hc} \left(\frac{x_1 + x_2}{2}\right) \leq\score_{\Hc}\left(\frac{x_3+ x_4}{2}\right)$.
\item if  $y_1\neq y_2$ and $ y_3 \neq y_4$. Then
$\score_{\Hc} \left(\frac{x_1 + x_2}{2}\right) \leq\score_{\Hc} \left(\frac{x_3+ x_4}{2}\right)$.
\item if  $y_1= y_2$ and $ y_3 = y_4$. Then
$\score_{\Hc} \left(\frac{x_1 + x_2}{2}\right) \geq\score_{\Hc}\left(\frac{x_3+ x_4}{2}\right)$.
\end{itemize}
\end{enumerate}
\end{prop}
The key conclusion drawn from these properties is that the midpoints
between the closest oppositely labeled neighboring examples have the highest score.  If there
are no oppositely labeled neighbors, then the score is largest at the
midpoint of the largest gap between consecutive samples. 
Thus, the
score results in a binary search for the thresholds definining the
classifier.  Using the proposition above, it is easy to show the
following result, proved in the Appendix~\ref{s:AppBinarySearch}.
\begin{corollary}
       Consider $N$ points uniformly distributed in the interval
       $[0,1]$ labeled according to a $k$-piecewise constant function $g(x)$ so that $y_i=g(x_i)\in\{-1,+1\}$ and length of the pieces are roughly on the order of $\Theta(1/K)$. 
	Then by running the proposed active learning algorithm with
        Laplace Kernel and any bandwidth, after $O(k\log N)$ queries
        the sign of the resulting interpolant $f$ correctly labels all
        $N$ examples (i.e., the training error is zero).
\label{cor:binsearch}
\end{corollary}

This statement is true for $N>5/h$. The proof is provided in Appendix~\ref{s:AppBinarySearch}.

\vspace{-.05in}

\subsection{General Radial-Basis Kernels in One Dimension}
\label{s:RadialKernelOneD}
In the next proposition, 
we look at the special case of radial basis kernels, defined in Equation\eqref{eq:kernelDef} applied to one dimensional functions with only three initial points. 
We  show how maximizing $\score_{\Hc}$ with the appropriate Hilbert norm is equivalent to picking the zero-crossing point of our current interpolator.
\begin{prop}[One Dimensional Functions with Radial Basis Kernels]
\label{p:RadialBasis}
Assume that  for any pair of samples $x,x'\in\Lc$ we have $|x-x'|\geq \Delta$. 
Assume $\Delta h^{-1/p} \geq D$ for a constant value of $D$. 
 Let $x_1< x_2<x_3\in\mathbb{R}$, $y_1=y_2 = +1$ and $y_3=-1$. For $\xt$ such that $x_2+\Delta/2<\xt<x_3-\Delta/2$, we have
 $\score_{\Hc}(\xt) \leq \score_{\Hc}(u^*)$
%
 %
 where $u^*$ is the point satisfying $f(u^*)=0$. 
\end{prop}
The proof is rather tedious and appears in  Appendix~\ref{s:ProofFirstPoint}. But the idea is based on showing that with small enough bandwidth, $\|f^{\xt}_+\|$ is increasing in $\xt$ in the interval $[x_2+\Delta/2 , x_3-\Delta/2]$  and  $\|f^{\xt}_-\|$ is decreasing in $\xt$ in the same interval. This shows that $\max_{\xt}\min_{\es\in\{-1,+1\}} \|f^{\xt}_{\es}\|$ occurs at $u^*$ such that 
$\|f^{\xt^*}_+\|=\|f^{\xt^*}_-\|$. We showed that this is equivalent to the condition $f(u^*)=0$. 
%
%

\input{propertiesDB}

%% file: propertiesDB.tex
\subsection{Properties of data based-norm criterion}
\label{s:result}

Intuitively, $\score_{\data}$ measures the expected change in the squared norm over
all unlabeled examples if $u \in {\cal U}$ is selected as the next
point. This norm is sensitive to the particular distribution of the
data, which is important if the data are clustered.  This behavior will
be demonstrated in the multidimensional setting discussed next.

In this section, we present two theoretical results on the properties of data-based norm selection criterion. 
To do so, we will prove the properties of the selected examples based on the data-based norm in the context of  the clustered data. 
In particular, if the support of the generative distribution $P_X(x)$ is composed of several disjoint clusters, the data-based norm criterion prioritizes labeling samples from bigger clusters first. Subsequently, it selects a sample from each cluster to be labeled.
 If the clustering in the dataset is aligned with their labels (most of the samples in the same cluster are in the same class), labeling one sample in each cluster ensures rapid decay in the probability of error of the classifier as a function of number of labeled samples. This behavior is consistent with numerical simulations presented in Section~\ref{s:numsim}.

The next theorem will show that if the clusters are well-separated (the distance between the clusters are sufficiently large), then the first example to be selected to for labeling is in the biggest cluster. 
\begin{theorem}[First point in clustered data]\label{t:FirstPoint}
Fix $p>1$ and $h>0$. Let the distribution $\mathbb{P}(X)$ be uniform over $M$ disjoint sets $B_1,\cdots,B_M$ such that $B_i$ is an $\ell_p$ ball with radius $r_i$ and center $c_i$, \textit{i.e.}, 
\begin{align}
\label{eq:ballDef}
B_i= \Bc_{d,p}(r_i;c_i) := \{x\in\mathbb{R}^d: \|x-c_i\|_{p}\leq r_i \}\,.
\end{align}
Without loss of generality, assume $r_1> r_2>\cdots>r_K$. 
Define
$D=\min_{i\neq j} \|c_i-c_j\|_p -2r_1$ as an upper bound for the minimum distance between the clusters.

Assume $\Lc=\varnothing$ and let the interpolating functions $f$ be defined in~\eqref{eq:InterpFnc} with $k_{h,p}$ (defined in~\eqref{eq:kernelDef}). The selection criterion is based on  the $\score_{\data}$ function defined in~\eqref{dscore}. 
If 
$$D>  \frac{h}{2}\Big[\ln M - \ln
\big(1-(r_2/r_1)^d\big)\Big]\quad \quad \text{and}\quad \quad r_1\leq h/2,$$ 
 then the first point to be labeled is in the biggest ball, $B_1$.
\end{theorem}
The proof is presented in Appendix~\ref{s:ProofFirstPoint}.

The next theorem shows that if the distance between the clusters are sufficiently large and the radius of the clusters are not too large, then the active learning algorithm based on the notion of $\score$ with data-based norm labels one sample from each cluster before zooming in inside the clusters. 

\begin{theorem}[Cluster exploration]\label{t:clusterExplore}
Let $\Sc$ be the support of $P_X$. 
Assume $\Sc=\cup_{i=1}^{M} B_i$ where $B_i$'s are $\ell_p$-balls with radii $r$ and centers $c_i$. 
Define $D:=\min_{i\neq j} \|c_i-c_j\|_p - 2r_1$ to be the minimum distance between the clusters.
Let $\Lc=\{x_1,x_2,\cdots,x_L\}$ be $L<M$ labeled points such that $x_1\in B_1, x_2\in B_2,\cdots,x_L\in B_L$. 
Let the selection criterion be based on  the $\score_{\data}$ function defined in~\eqref{dscore}. 
If 
 $r<h/3$ and $D\geq 12 h \ln(2M)$, then the next point to be labeled
 is in  a new ball ($\cup_{i={L+1}}^MB_{i}$) containing no labeled points.
\end{theorem}

As a corollary of the above theorem, one can see that if the ratio of the distance between the clusters to the radius of clusters is sufficiently large ($D/r>36 ln(2M)$), then one can use a kernel with proper bandwidth which picks one sample from each cluster initially. 
The proof is presented in Appendix~\ref{s:prooftclusterExplore}.

%% file: numsimDBRKHS.tex
\section{Numerical Simulations of kernel based }
\label{s:numsim}
In this Section, we present the outcome of numerical simulations of the proposed selection criteria on synthetic and real data. 
In this section, $\score_{\Hc}$ is used to denoted the $\score$ function defined in~\eqref{fscore} with the Hilbert norm associated with the Laplace Kernel.  Similarly, $\score_{\data}$ is the $\score$ function defined in~\eqref{dscore} with the data-based norm.

\input{numsimKernel}

\subsection{Multidimensional setting with clustered data}
To capture the properties of the proposed selection criteria in clustered data, we implemented the algorithm on synthetic clustered data in Figures~\ref{fig:cluster} and~\ref{F:TwoDClustering}.
 We demonstrate how
the data-based norm also tends to automatically select representive
examples from clusters when such structure exists in the unlabeled
dataset. Figure~\ref{fig:cluster} compares the behavior of selection
based on $\score_{\Hc}In $ with the RKHS norm and $\score_{\data}$ with
data-based norm, when data are clusters and each cluster is
homogeneously labeled.  We see that the data-based norm quickly
identifies the clusters and labels a representative from each, leading
to faster error decay as shown on the right.
\begin{figure}[h]
        \centering
\includegraphics[width=10cm]{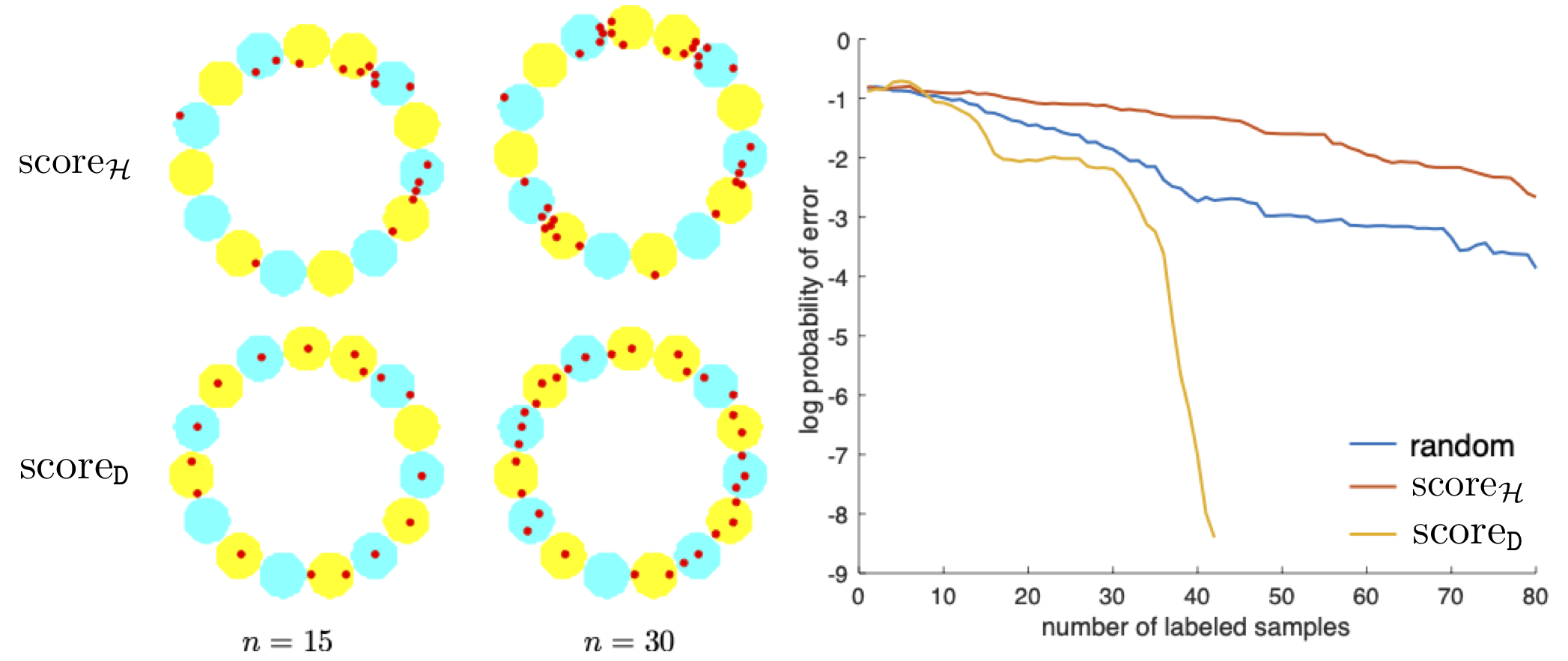}
        \caption{Uniform distribution of samples, smooth boundary,
          Laplace Kernel, Bandwidth$ = .1$. On left, sampling behavior
          of $\score_{\Hc}$ and $\score_{\data}$ at progressive stages
          (left to right).  On right, error probabilities as a function
          of number of labeled examples.}
\label{fig:cluster}
\end{figure}

\begin{figure}[h]
 \centering
        \begin{subfigure}[b]{0.55\textwidth}
\includegraphics[width=\linewidth]{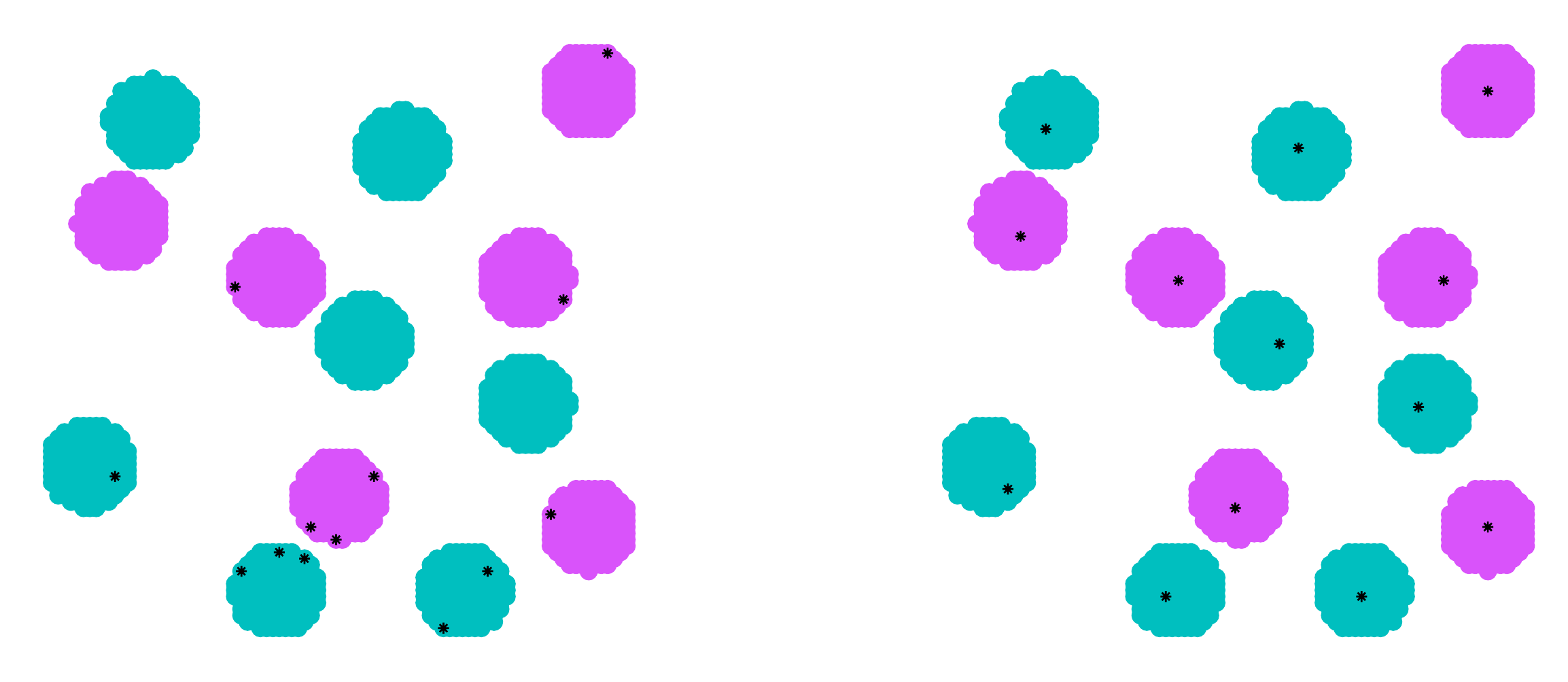}
    \end{subfigure}
   ~ 
        \caption{
            Points in blue
            and yellow clusters are labeled $+1$ and $-1$,
            respectively.
            The left figure uses $\score_{\Hc}$ to be the $\score$ function defined in~\eqref{fscore} with the Hilbert norm associated with the Laplace Kernel.  Similarly, $\score_{\data}$ is the $\score$ function defined in~\eqref{dscore} with the data-based norm.  The first 13 samples selected by
            $\score_{\Hc}$ and $\score_{\data}$ are depicted as black
            dots. $\score_{\data}$ has labeled one sample from each
            cluster, but $\score_{\Hc}$ has not labeled any samples
            from 5 clusters.  Note that $\score_{\Hc}$ has spent some
            of the sample budget to discriminate between nearby
            clusters with opposite labels.  }
\label{F:TwoDClustering}
\end{figure}
In the setup in Figure~\ref{F:TwoDClustering}, the samples are generated based on a uniform distribution on 13 clusters. 
Points in blue  and yellow clusters are labeled $+1$ and $-1$, respectively.
We run the two variations of proposed active learning algorithms and compare their sampling strategy in this setup. 
The left figure uses $\score_{\Hc}$ to be the $\score$ function defined in~\eqref{fscore} with the Hilbert norm associated with the Laplace Kernel.  Similarly, $\score_{\data}$ is the $\score$ function defined in~\eqref{dscore} with the data-based norm. 

The selection criterion based on $\score_{\Hc}$ prioritizes sampling on the decision boundary of the current classifier where the currently oppositely labeled samples are close to each other. 
This behavior of the algorithm based on  $\score_{\Hc}$ in one dimension is proved in Sections~\ref{s:LaplaceKernelOneD} and~\ref{s:RadialKernelOneD}.
Alternatively, $\score_{\data}$ prioritizes labeling at least one sample from each cluster. Hence, after labeling 13 samples, the active learning algorithm based on $\score_{\data}$  has one sample in each cluster, but the active learning algorithm based on $\score_{\data}$ has not labeled any samples in 5 clusters.

\subsection{MNIST experiments}

Here we illustrate the performance of the proposed active learning
method on the MNIST
dataset. 
We ran algorithms based on our proposed  selection criteria for a binary classification task on MNIST dataset. 
The binary classification task used in this experiment assigns a label $-1$ to any digit in set $\{0,1,2,3,4\}$ and label $+1$ to $\{5,6,7,8,9\}$.
The goal of the classifier is detecting whether an image belongs to the set of numbers greater or equal to $5$ or not.
We used Laplace kernel as defined in~\eqref{eq:kernelDef} with $p=2$ and $h=10$ on the vectorized
version of a dataset of $1000$ images. 
In Figures~\ref{F:MNISTPrerroTraining}, $\score_{\Hc}$ is the $\score$ function defined in~\eqref{fscore} with the Hilbert norm associated with the Laplace Kernel.  Similarly, $\score_{\data}$ is the $\score$ function defined in~\eqref{dscore} with the data-based norm. 

To asses the quality of performance of each of the selection criteria, we compare the probability of error of the interpolator at each iteration.  
In particular,  we plot the probability of error of the interpolator  as a function of number of labeled samples, using the $\score_{\Hc}$ and $\score_{\data}$ functions  on the training set and test set separately.
 For comparison, we also plot the probability of error when the selection criterion for picking samples to be labeled is random. 

Figure~\ref{F:MNISTPrerroTraining} (a) shows the decay of probability of error in the training set. 
When the number of labeled samples is equal to the number of samples in the training set, it means that all the samples in training set are labeled and used in constructing the interpolator. 
 Hence, the probability of error on the training set for any selection criterion  is zero when number of labeled samples is equal to the number of samples in the training set.
Figure~\ref{F:MNISTPrerroTraining} (b) shows the probability of error on the test set as a function of the number of labeled samples in the training set selected by each selection criterion. 

\begin{figure}[h]
 \centering
        \begin{subfigure}[b]{0.49\textwidth}
        \centering
\includegraphics[width=.7\linewidth]{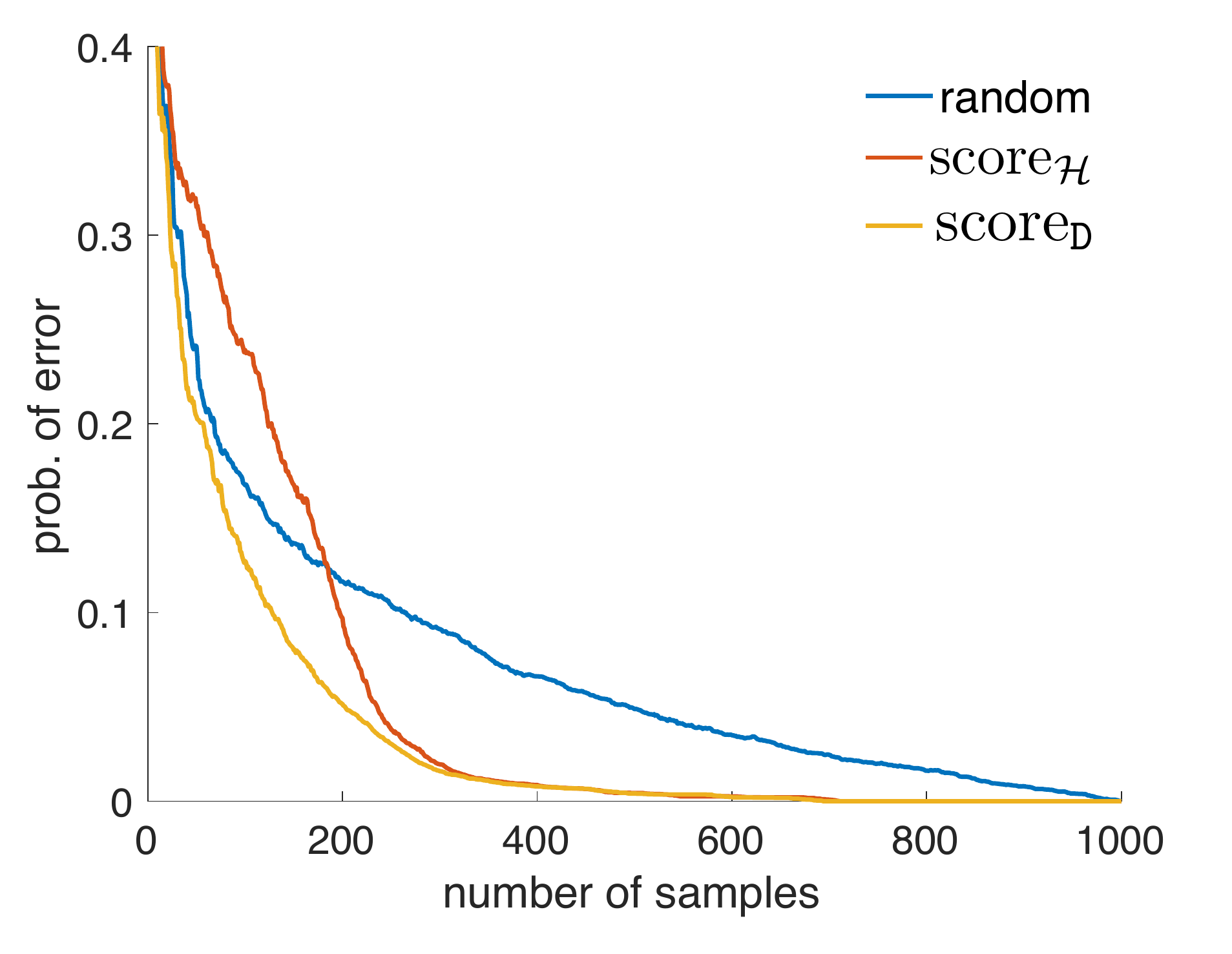}
    \end{subfigure}
            \begin{subfigure}[b]{0.49\textwidth}
        \centering
\includegraphics[width=.7\linewidth]{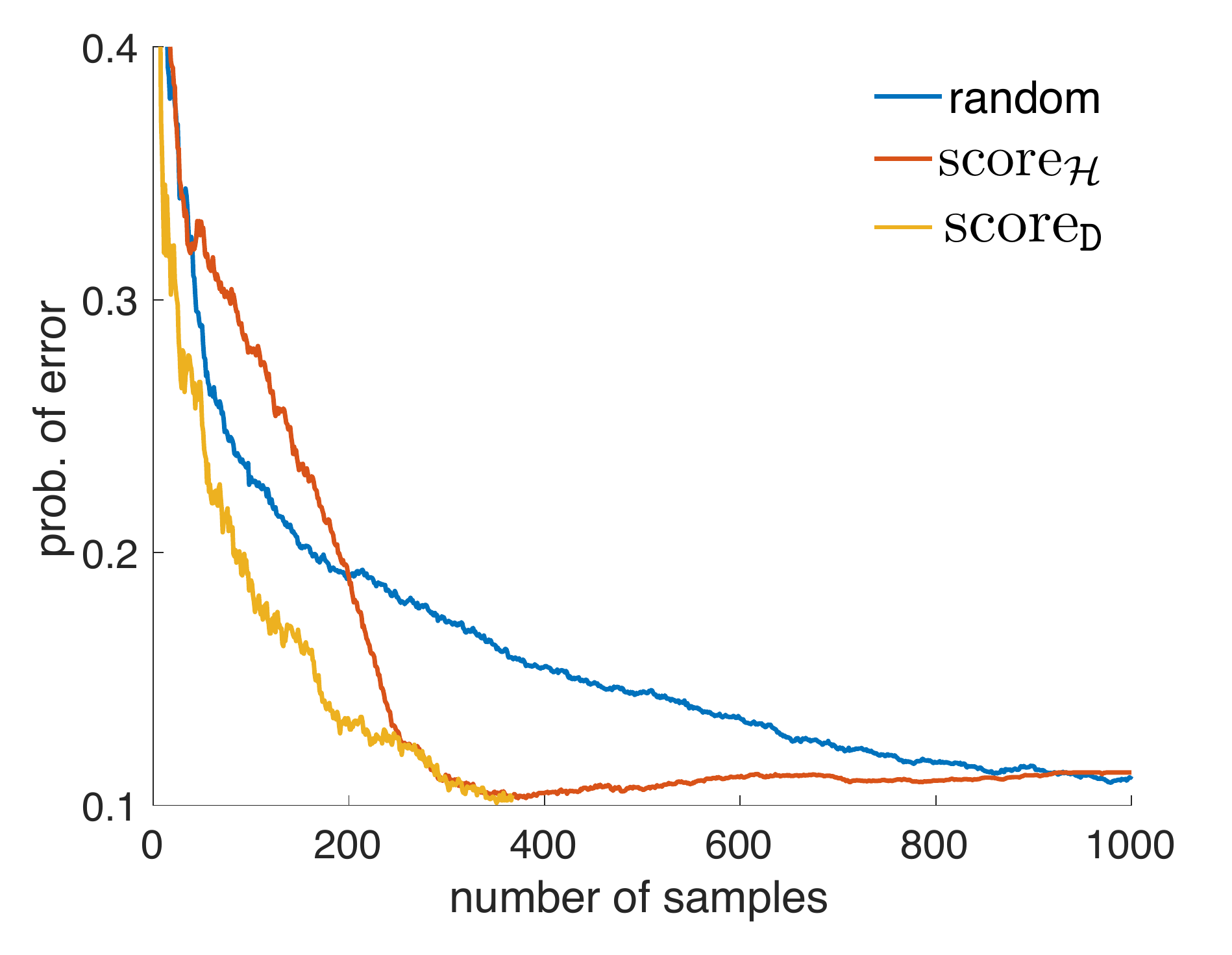}
    \end{subfigure}
   ~\caption{Probability of error for learning a classification task on MNIST data set. The performance of three selection criteria for labeling the samples: random selection, active selection based on $\score_{\Hc}$, and active selection based on $\score_{\data}$. The first curve depicts the probability of error on the training set and the second curve is the probability of error on the test set. 
}
\label{F:MNISTPrerroTraining}
\end{figure}

\subsubsection{Clustering in MNIST}
The binary classification task used in the MNIST experiment assigns a label $-1$ to any digit in set $\{0,1,2,3,4\}$ and label $+1$ to $\{5,6,7,8,9\}$.
 We expect that the images are clustered where each cluster would correspond to the images of a digit. We expect that the advantageous behavior of using data-based norm criterion in clustered data is one of the reasons for faster decay of probability of error of the $\score_{\data}$ in Figure~\ref{F:MNISTPrerroTraining}. 

To verify this intuition, we look at the samples that were chosen by each criterion and the digit corresponding to that sample. 
Note that this digit is the number represented in the image and not the label of the sample since the label of each sample is $+1$ or $-1$ depending whether the number is greater than $4$ or not. 
After labeling $100$ samples, we look at histogram of the digits associated with the labeled samples with each criterion $\score_{\Hc}$ and $\score_{\data}$. 
If samples of each cluster are chosen to be labeled uniformly among clusters, we would see about $10$ labeled samples in each cluster. 
Figure~\ref{F:MNISTClustering}  shows the histogram described above for two variations of the selection criteria based on $\score_{\Hc}$ or $\score_{\data}$. We observe that selecting samples based on $\score_{\data}$ is much more uniform among the clusters. On the contrary, selecting samples based on $\score_{\Hc}$ gives much less uniform samples among clusters. In the particular example given in Figure~\ref{F:MNISTClustering}, we see that even after selecting $100$ samples to be labeled, no sample in the cluster of images of number $0$ has been labeled in this instance of execution of the selection algorithm based on notion of $\score_{\Hc}$. 

To quantify the uniformity of selecting samples in different clusters, we ran this experiment $20$ times and estimated the standard deviation of number of labeled samples in each cluster after labeling $100$ samples. Note that since we have $10$ clusters, the mean of the number of labeled samples in each cluster is $10$. The standard deviation using $\score_{\Hc}$ is $4.1$ whereas  standard deviation using $\score_{\Hc}$ is $2.7$. This shows that selection criterion based on $\score_{\data}$ samples more uniformly among the clusters.

\begin{figure}[h]
 \centering
        \begin{subfigure}[b]{0.49\textwidth}
        \centering
\includegraphics[width=.7\linewidth]{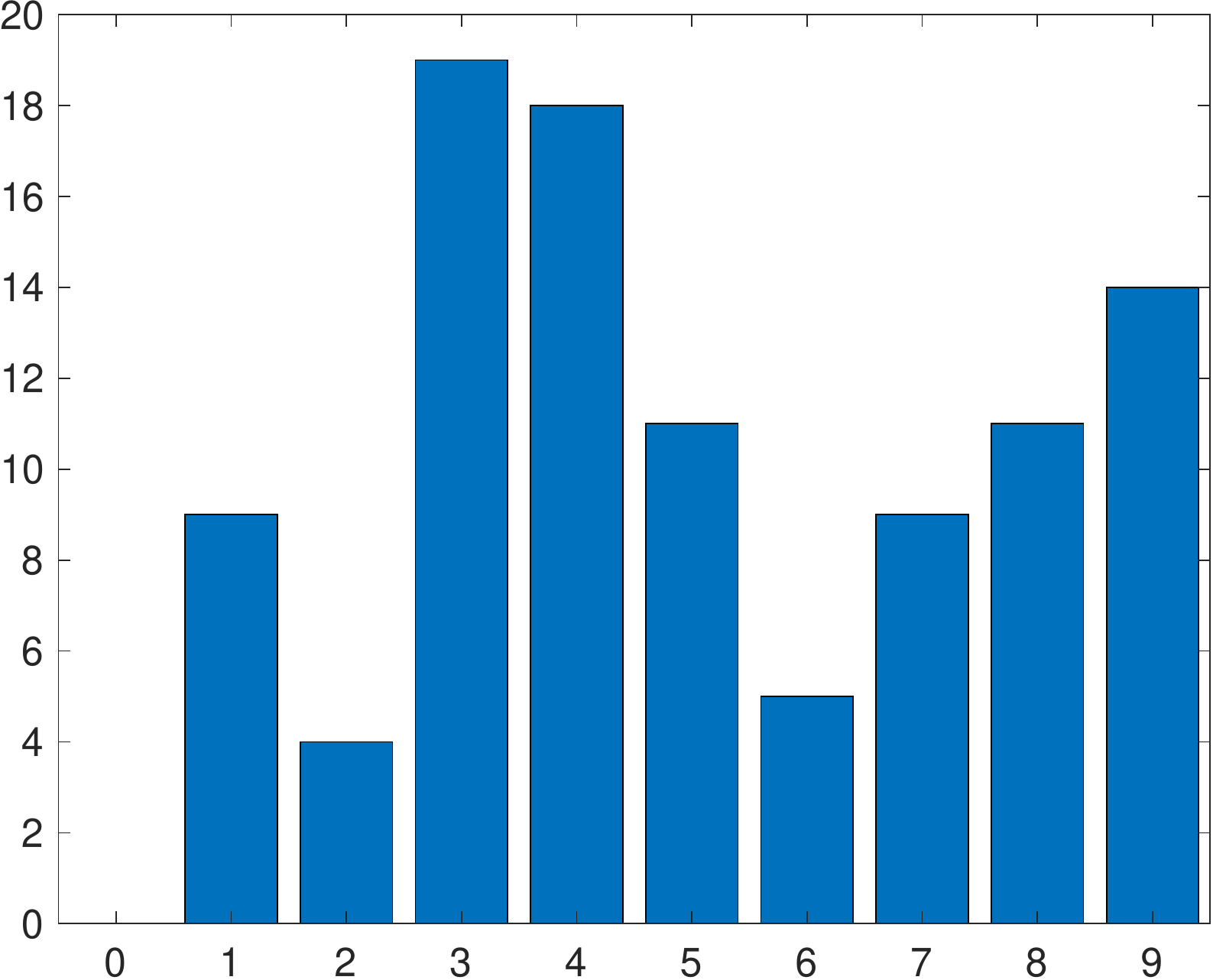}
    \end{subfigure}
            \begin{subfigure}[b]{0.49\textwidth}
        \centering
\includegraphics[width=.7\linewidth]{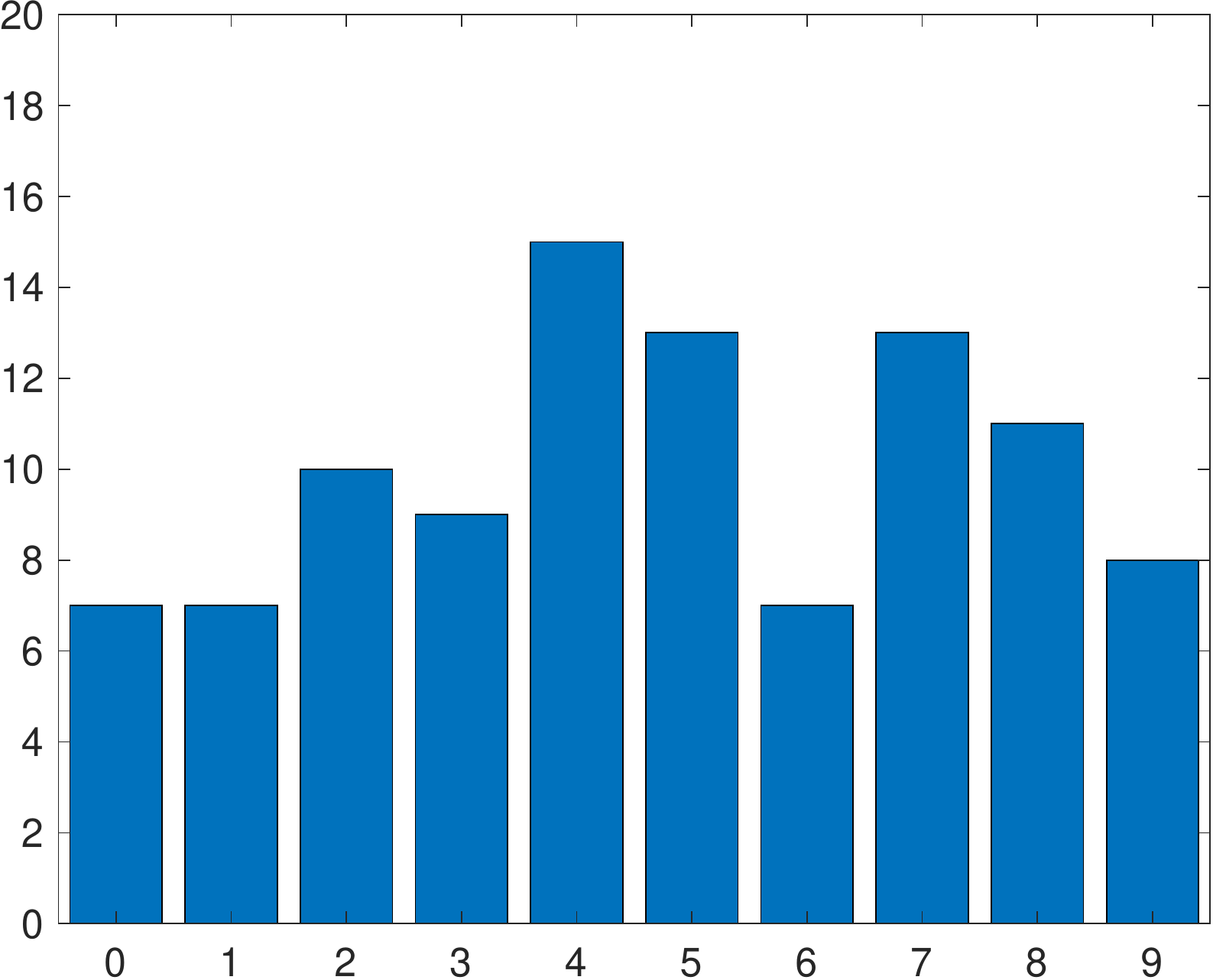}
    \end{subfigure}
   ~ 
        \caption{The histogram of the handwritten digits associated with the labeled samples
        after labeling $100$ samples. The first histogram is for the selection criterion  $\score_{\Hc}$ and the second histogram is for the selection criterion 
            $\score_{\data}$. Notably, $\score_{\Hc}$ has not labeled any of the images of the digit $0$.}
\label{F:MNISTClustering}
\end{figure}

%% file: numsimKernel.tex
\subsection{Bisection in one dimension}

The bisection process is illustrated experimentally in the
Figure~\ref{fig:oneDbs} below.  $\score_{\Hc}$ uses the RKHS
norm. For comparison, we also show the behavior of the algorithm using
$\score_{\data}$ and the data-based norm.  Data selection using either score drives the
error to zero faster than random sampling (as shown on the left). We
clearly see the bisection behavior of $\score_{\Hc}$, locating one
decision boundary/threshold and then another, as the proof corollary
above suggests. Also, we see that the data-based norm does more
exploration away from the decision boundaries.  As a result, the
data-based norm has a faster and more graceful error decay, as shown
on the right of the figure.  Similar behavior is observed in the
multidimensional setting shown in Figure~\ref{fig:SmoothCurveTwoD}.

\begin{figure}[h]
        \centering
\includegraphics[width=14cm]{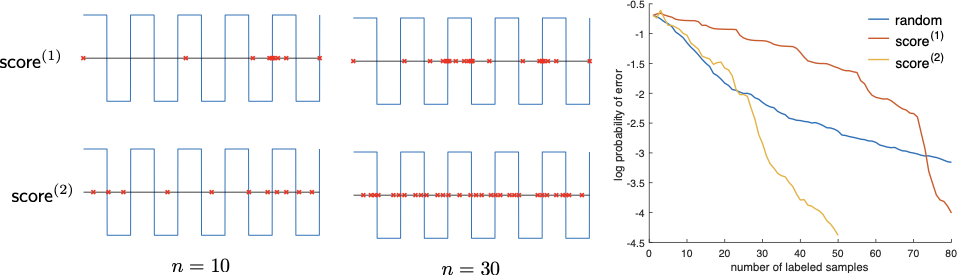}
        \caption{Uniform distribution of samples in unit interval,
          multiple thresholds between $\pm 1$ labels,  and active
          learning using Laplace Kernel, Bandwidth$ = .1$. Probability
          of error of the interpolated function shown on right.}
\label{fig:oneDbs}
\end{figure}

\subsection{Multidimensional setting with smooth boundary}

The properties and behavior found in the one dimensional setting carry
over to higher dimensions.  In particular, the max-min norm criterion
tends to select unlabeled examples near the decision boundary {\em
  and} close to oppositely labeled examples,  This is illustrated in
Figure~\ref{fig:HeatmapKernel} below.
The inputs points (training examples) are
uniformly distributed in the
square $[-1,1]\times[-1,1]$. We trained an Laplace kernel machine to perfectly interpolate four training
points with locations and binary labels as depicted in
Figure~\ref{fig:HeatmapKernel}(a).  The color depicts the magnitude of the
learned interpolating function: dark blue is $0$ indicating the
``decision boundary'' and bright yellow is approximately
$3.5$. Figure~\ref{fig:HeatmapKernel}(b) denotes the score for selecting a point
at each location based on RKHS norm criterion.  Figure~\ref{fig:HeatmapKernel}(c) denotes the score for
selecting a point at each location based on data-based norm criterion
discussed above. Both criteria select the point on the decision boundary, but the RKHS
norm favors points that are  closest
to oppositely labeled examples whereas the data-based norm favors points
on the boundary further from labeled examples.

\begin{figure}[h]
        \centering
\includegraphics[width=.6\linewidth]{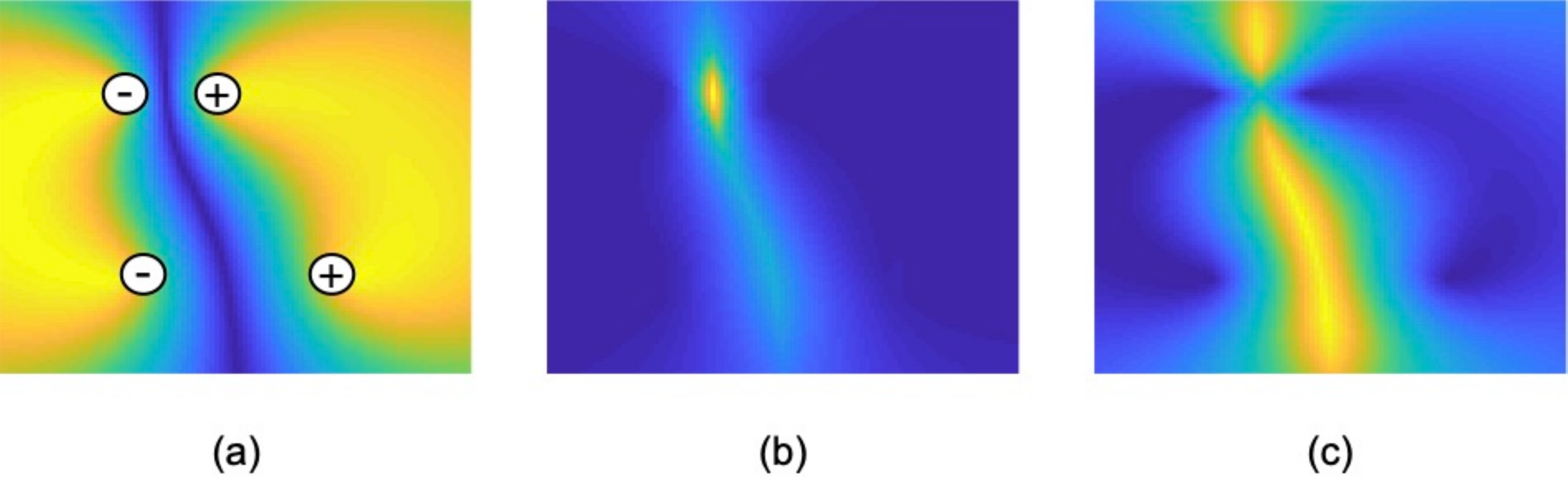}
        \caption{Data selection of Laplace kernel active learner. (a)
  Magnitude of output
  map kernel machine trained to interpolate four data points as
indicated (dark blue is $0$ indicating the learned decision boundary).
(b) Max-Min RKHS norm selection of next point to label. Brightest yellow is
location of highest score and selected example. (c) Max-Min selection of next point to label using data-based
norm.  Both select the point on the decision boundary, but the RKHS
norm favors points that are  closest
to oppositely labeled examples.}
\label{fig:HeatmapKernel}
\end{figure}
\vspace{-.05in}

Next we present a modified scenario in which the examples are not
uniformly distributed over the input space, but instead concentrated
only in certain regions indicated by the magenta highlights in
Figure~\ref{fig:HeatmapKernelBoxed}(a).  In this setting, the example selection
criteria differ more significantly for the two norms.  The weight norm selection
criterion remains unchanged, but is applied only to regions where
there are examples. Areas with out examples to select are indicated by dark
blue in Figure~\ref{fig:HeatmapKernelBoxed}(b)-(c). The data-based norm is
sensitive to the non-uniform input distribution, and it scores
examples near the lower portion of the decision boundary highest.

\begin{figure}[h]
        \centering
\includegraphics[width=.6\linewidth]{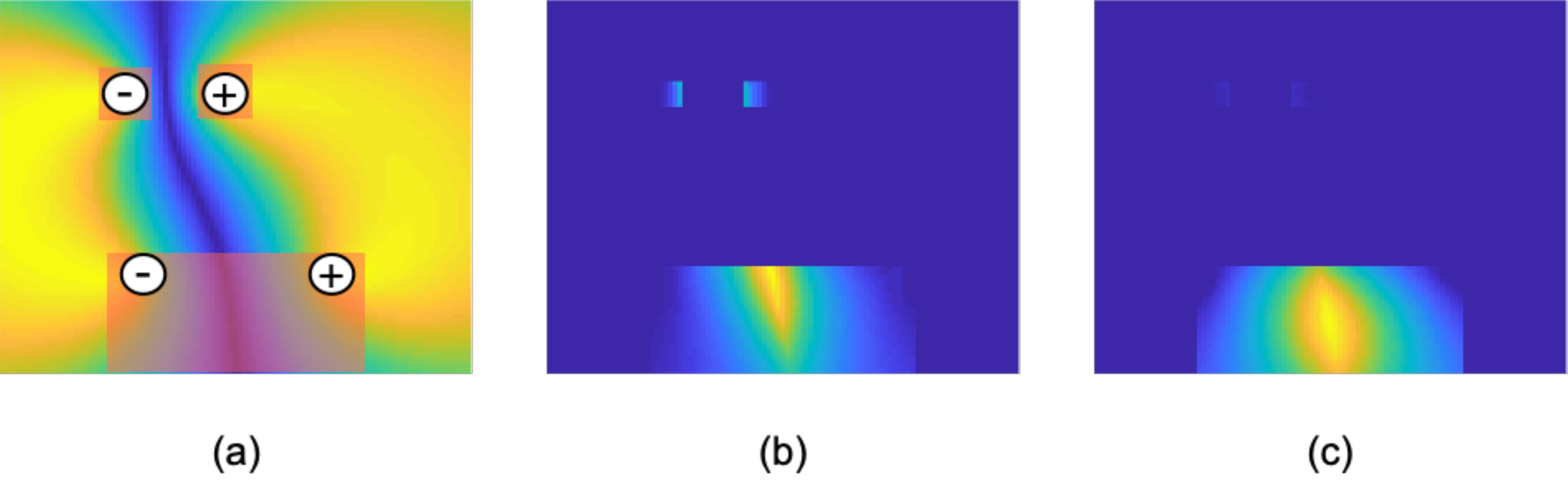}
        \caption{Data selection of Laplace kernel active learner. (a)
  Unlabeled examples are only available in magenta shaded regions.
(b) Max-Min selection map using RKHS
norm~\eqref{fscore}. (c) Max-Min selection map using data-based
norm defined in Equation~\eqref{dscore}.}
\label{fig:HeatmapKernelBoxed}
\end{figure}

The distinction between the max-min selection criterion using the RKHS
vs.\ data-based norm is also apparent in the experiment in which a
curved decision boundary in two dimensions is actively learned using a
Laplace kernel machine, as depicted in
Figure~\ref{fig:SmoothCurveTwoD} below.  $\score_{\Hc}$ is the max-min RKHS norm criterion at progressive stages of the
learning process (from left to right).  The data-based norm is used in
$\score_{\data}$ defined in Equation~\eqref{dscore}.  Both dramatically
outperform a passive (random sampling) scheme and both demonstrate how
active learning automatically focuses sampling near the decision
boundary between the oppositely labeled data (yellow vs.\ blue).
However, the data-based norm does more exploration away from the
decision boundary.  As a result, the data-based norm requires slightly
more labels to perfectly predict all unlabeled examples, but has a
more graceful error decay, as shown on the right of the figure.

\begin{figure}[h]
        \centering
\includegraphics[width=12cm]{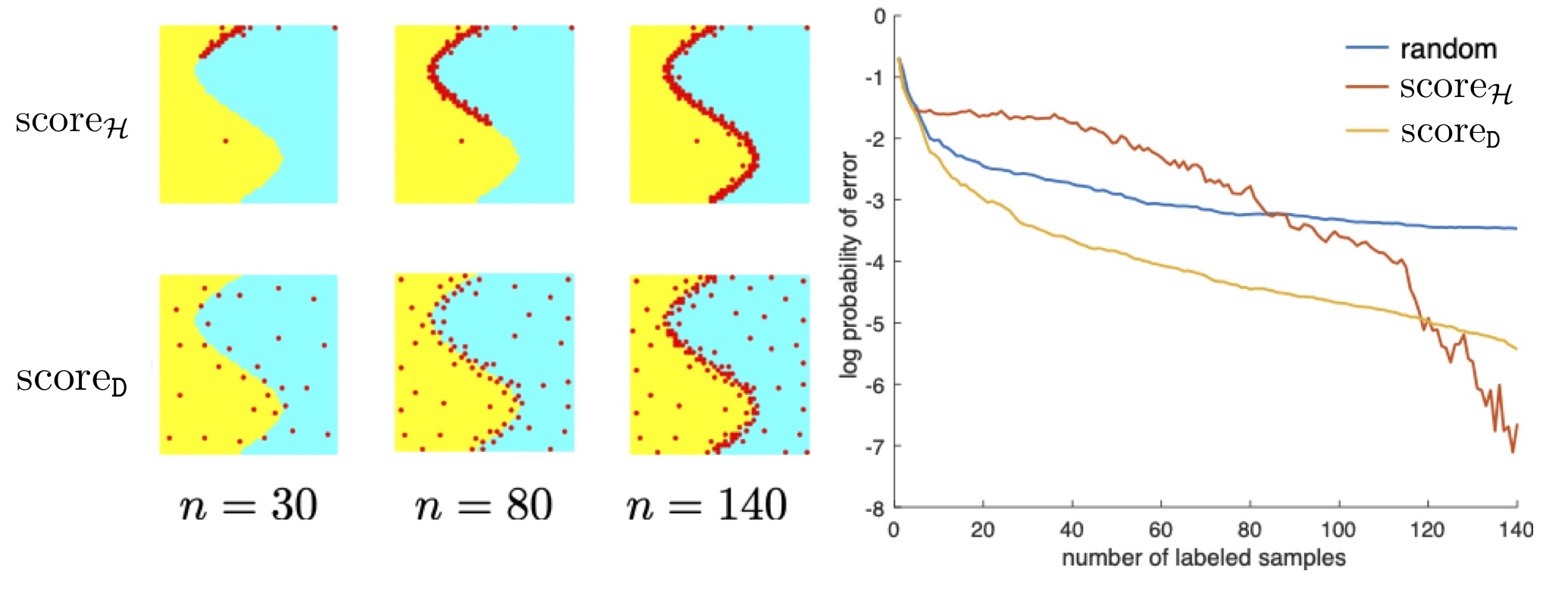}
        \caption{Uniform distribution of samples, smooth boundary,
          Laplace Kernel, Bandwidth$ = .1$.  On left, sampling behavior
          of $\score_{\Hc}$ and $\score_{\data}$ at progressive stages
          (left to right).  On right, error probabilities as a function
          of number of labeled examples.}
\label{fig:SmoothCurveTwoD}
\end{figure}

%% file: NN.tex

\section{Interpolating Neural Network Active Learners}

Here we briefly examine the extension of the max-min criterion and its
variants to neural network learners.  Neural network complexity or
capacity can be controlled by limiting magnitude of the network
weights
\cite{bartlett1997valid,neyshabur2014search,zhang2016understanding}. A
number of weight norms and related measures have been recently proposed in the literature
\cite{neyshabur2015norm,bartlett2017spectrally,golowich2018size,arora2018stronger,neyshabur2017pac}.
For example, ReLU networks with a single hidden layer
and minimum $\ell_2$ norm weights coincide with linear spline
interpolation \cite{savarese2019infinite}.  With this in mind, we
provide empirical evidence showing that defining the max-min criterion
with the norm of the network weights yields a neural network active
learning algorithm with properties analagous to those obtained in the
RKHS setting. 

Consider a single hidden layer network with ReLU activation units
trained using MSE loss.  In Figure~\ref{fig:nn1} we show the results
of an experiment implemented in PyTorch in the same settings considered
above for kernel machines in Figures~\ref{fig:HeatmapKernel} and \ref{fig:HeatmapKernelBoxed}. We trained an overparameterized network
with $100$ hidden layer units to perfectly interpolate four training
points with locations and binary labels as depicted in
Figure~\ref{fig:nn1}(a).  The color depicts the magnitude of the
learned interpolating function: dark blue is $0$ indicating the
``decision boundary'' and bright yellow is approximately
$3.5$. Figure~\ref{fig:nn1}(b) denotes the $\score_{\Hc}$ with the
weight norm (i.e., the $\ell_2$
norm of the resulting network weights when a new sample is selected at
that location).  The brightest yellow indicates the highest score and the
location of the next selection. Figure~\ref{fig:nn1}(c) denotes the
$\score_{\data}$ with the data-based norm defined in Equation~\eqref{dscore}.
In both cases, the max occurs at roughly the same location, which is
near the
current decision boundary {\em and} closest to oppositely labeled
points.  The data-based norm also places higher scores on points further away from the
labeled examples. Thus, the data selection behavior of the neural network is
analagous to that of the kernel-based active learner (compare with\ Figure~\ref{fig:HeatmapKernel}).

\begin{figure}[h]
\centerline{\includegraphics[width=9cm]{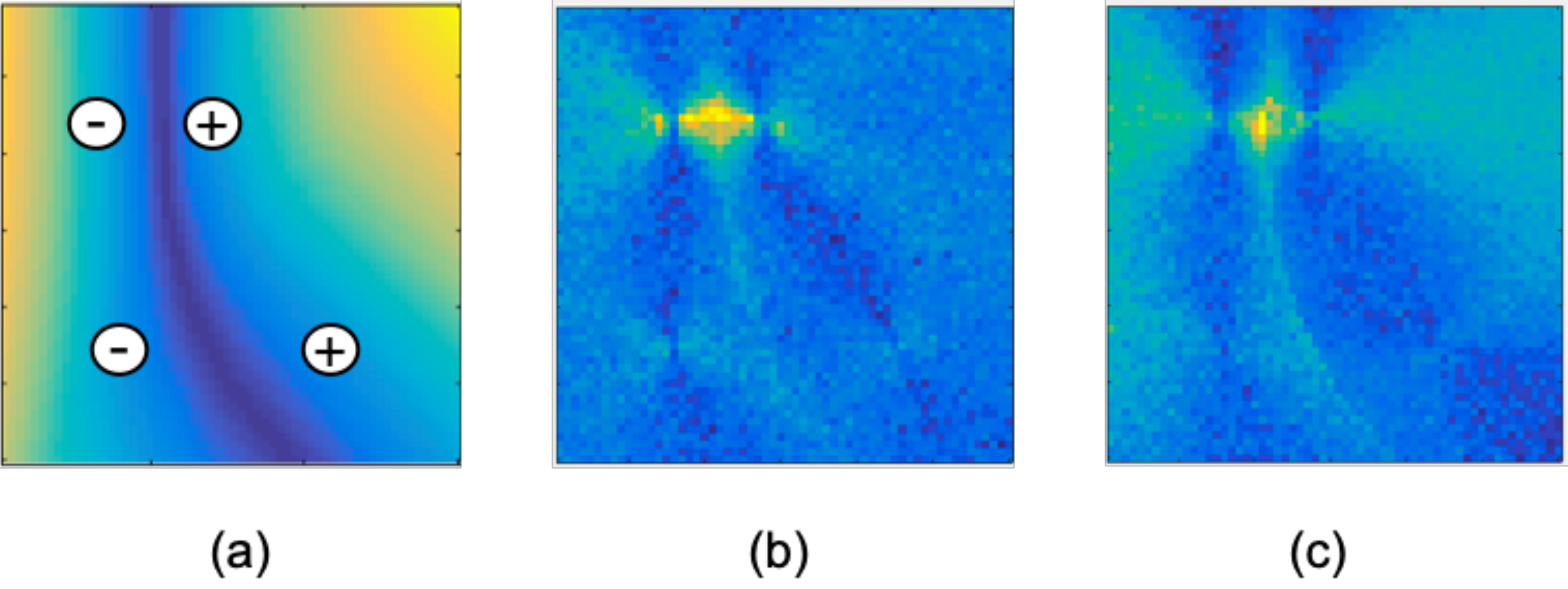}} 
\caption{Data selection of neural network active learner. (a)
  Magnitude of output
  map of single
  hidden layer ReLU network trained to interpolate four data points as
indicated (dark blue is $0$ indicating the learned decision boundary).
(b) Max-Min selection of next point to label using network weight
norm. (c) Max-Min selection of next point to label using data-based
norm.  Both select the point on the decision boundary that is closest
to oppositely labeled examples.}
\label{fig:nn1}
\end{figure}

Next we present a modified scenario in which the examples are not
uniformly distributed over the input space, but instead concentrated
only in certain regions indicated by the magenta highlights in
Figure~\ref{fig:nn2}(a).  In this setting, the example selection
criteria differ more significantly for the two norms.  The weight norm selection
criterion remains unchanged, but is applied only to regions where
there are examples. Areas without examples to select are indicated by dark
blue in Figure~\ref{fig:nn2}(b)-(c). The data-based norm is
sensitive to the non-uniform input distribution, and it scores
examples near the lower portion of the decision boundary highest.
Again, this is quite similar to the behavior of the kernel active
learner  (compare with\ Figure~\ref{fig:HeatmapKernelBoxed}).

\begin{figure}[h]
\centerline{\includegraphics[width=9cm]{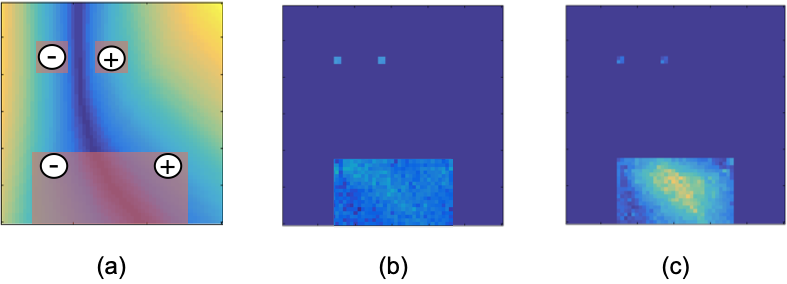} }
\caption{Data selection of neural network active learner. (a)
  Unlabeled examples are only available in magenta shaded regions.
(b) Max-Min selection map using network weight
norm. (c) Max-Min selection map using data-based
norm.}
\label{fig:nn2}
\end{figure}

%% file: proofNN.tex
\section{Maximin active learning with neural networks}

  We present the proof of Theorems~\ref{thm:BinarySearchLinearSpline} and~\ref{thm:DataSearchLinearSpline} assuming the solutions to~\eqref{eq:defF} and~\eqref{eq:defFtu} -- minimum norm interpolating functions -- are linear spline functions with
  knots at each data point. According to Theorem~\ref{t:minnormNN}, there are other solutions to the minimum
  ``norm'' neural network; but since $\score^{(1)}_{\spline}$ only
  depends on the ``norm'' it suffices to just consider the spline
  case. 
  Moreover, as shown in
  \cite{parhi2019minimum}, the weight norm $\|f\|_{\Fc}:= \|\bw\|_2$
  is equal to the total variation of $f$.  The total variation of a
  linear spline is the sum of the absolute values of the slopes of
  each linear piece.  We use this equivalence throughout the proof.

For a linear spline function $f(x)$ with  knots at each data point,  
 the assumption  $f(\infty)=f(-\infty)=0$ guarantees that for any $x_i$ such that $y_i=+1$, we have 
 \begin{align*}
 \lim_{x\to x_i^{-}} \frac{\mathrm{d}}{\mathrm{d}x}f(x)
 &\geq 0,\quad \quad
 \lim_{x\to x_i^{+}} \frac{\mathrm{d}}{\mathrm{d}x}f(x)
 \leq 0
 \end{align*}
 Similarly, for any $x_i$ such that $y_i=+1$, we have
 \begin{align*}
 \lim_{x\to x_i^{-}} \frac{\mathrm{d}}{\mathrm{d}x}f(x)
 &\leq 0,\quad \quad
 \lim_{x\to x_i^{+}} \frac{\mathrm{d}}{\mathrm{d}x}f(x)
 \geq 0
 \end{align*}
 The assumption  $f(\infty)=f(-\infty)=0$ also guarantees that $f'(\infty)=f'(-\infty)=0$.
 
Also, if a point $u$ is between two labeled points $x_1$ and $x_2$ such that $x_1<x_2$, then for any label $\es\in\{-,1+1\}$ we know that $f^{u}_{\es} (x)=f(x)$ for $x\leq x_1$ or $x\geq x_2$. 
Using these properties, we find the maximizer of $\score_{\Fc}(u)$ and $\score_{\data}(u)$ in various cases. 


\subsection{Maximin Criterion with Function Norm and Neural Networks (Proof of Theorem~\ref{thm:BinarySearchLinearSpline})}
\label{s:proofNN}
To prove Theorem~\ref{t:minnormNN},~\cite{nn-linear-spline} shows that optimizing the parameters of a two layer Neural Network as described in Theorem~\ref{t:minnormNN} to find $\min_{\mathbf{w}} \|\mathbf{w}\|_2$ such that $f(x_i)=y_i$ for all $i$ is the same as minimizing the $R(f)$ such that $f(x_i)=y_i$ for all $i$ where $R(f)$ is defined as
\begin{align*}
R(f) = \max\Big(
\big|f'(+\infty) +f'(-\infty)\big|,\,
\int |f''(x)| \,\mathrm{d}x
\Big)\,.
\end{align*}
Hence, we can use $R(f)$, as a proxy for a function norm in the context of Neural Networks. 

In our setup, adding two artificial points to the left and right with the same labels as the true end points ensures that the function $R(f)=\int |f''(x)| \,\mathrm{d}x$ since  $f'(+\infty)=f'(-\infty) =0$ for the minimum norm interpolating functions. 

Hence, for a set of points $\{(x_i,y_i)\}_i$, the norm of minimum norm interpolating function is the same as $R(f)=\int |f''(x)| \,\mathrm{d}x$, i.e., ``the summation of changes in the slope of minimal knot linear spline interpolation of points''. This quantity is used to compute the score of each unlabeled point. 
Note that clearly to compute the score of a point $u$ such that $x_j<u<x_{j+1}$, we only need to compute how much the slope of interpolating function changes by adding $u$ in the interval $x_j$ to $x_{j+1}$. In particular, 
\[
\|f^u_{\es}(x)\|_{\Fc} -
\|f(x)\|_{\Fc}
= 
2\frac{1-\es y_j}{u-x_j}
+
2\frac{1-\es y_{j+1}}{x_{j+1}-u}
-
2\frac{1-y_jy_{j+1}}{x_{j+1}-x_j}\,.
\]

\begin{enumerate}
\item 
Without loss of generality, assume $y_1=+1$ and $y_2= -1$. 
Hence, 
\[\lim_{x\to x_1^{-}} \frac{\mathrm{d}}{\mathrm{d}x}f(x)\geq 0\]
and
\[\lim_{x\to x_2^{+}}  \frac{\mathrm{d}}{\mathrm{d}x} f(x)\geq 0\,.\]
Looking at $f(x)$, we also have the slope of function between $x_1$ and $x_2$ to be $-\frac{2}{x_2-x_1}$. 

The same statements hold for $f_{\es}^u(x)$ for any $x_1<u<x_2$ and $\es\in\{-1,+1\}$.
\begin{align*}
\lim_{x\to x_1^{-}} \frac{\mathrm{d}}{\mathrm{d}x} f_{\es}^u(x)
&\geq 0\,, 
\quad \quad
\lim_{x\to x_2^{+}}  \frac{\mathrm{d}}{\mathrm{d}x} f_{\es}^u(x)
\geq 0\,.
\end{align*}

For $\es=+1$, the slope of $f_{\es}^u(x)$ for $x_1<x<u$ is zero and  the slope of $f_{\es}^u(x)$ for $u<x<x_2$ is $-\frac{2}{x_2-u}$. Since $f^{u}_{+} (x) = f(x)$ for $x\leq x_1$ or $x\geq x_2$, 
\begin{align*}
\|f(x)\|_{\spline}   - \frac{4}{x_2-x_1}
&= 
\|f^u_{+}(x)\|_{\spline}  - \frac{4}{x_2-u}\,.
\end{align*}
Using a similar calculation for $\es=-1$ we get
\begin{align*}
\|f^u_{+}(x)\|_{\spline} 
&= 
\|f(x)\|_{\spline}   - \frac{4}{x_2-x_1}
+ \frac{4}{x_2-u}\,,
\\
\|f^u_{-}(x)\|_{\spline} 
&= 
\|f(x)\|_{\spline}   - \frac{4}{x_2-x_1}
+ \frac{4}{u-x_1}\,.
\end{align*}
Hence, $\|f^u_{+}(x)\|_{\spline} \geq \|f^u_{-}(x)\|_{\spline} $ if and only if $u\geq  \frac{x_1+x_2}{2}$ and
\begin{align}
\label{eq:score1Su}
\score_{\spline} (u)= \|f^u(x)\|_{\spline} 
&= 
\|f(x)\|_{\spline}   - \frac{4}{x_2-x_1}
+ \min\Big\{\frac{4}{x_2-u}, \frac{4}{u-x_1}\Big\}\,.
\end{align}
This gives 
$$u^*:=
\argmax_{x_1<u\leq x_2}  \score^{(1)}_{\spline}(u)
= 
\frac{x_1+x_2}{2}\,.$$
and 
$$\score_{\spline} (u^*) =  \|f(x)\|_{\spline}   + \frac{4}{x_2-x_1}\,.$$
\item
Above computation shows that pair of neighboring oppositely labeled points $x_1<x_2$ and $x_3<x_4$,
if $x_2-x_1\geq x_4-x_3$, then 
\[\score_{\spline} \big(\frac{x_3+x_4}{2}\big) 
\geq 
\score_{\spline} \big(\frac{x_1+x_2}{2}\big)\,.\]
\item 
Assume $y_5=y_6=y$. Then, for $x_5\leq v\leq x_6$ we have $f^v_{y}(x) = f(x)$ for all $x$.
For $\es=-y$, we have 
\[
\|f^v_{-y}(x)\|_{\spline} = 
\|f(x)\|_{\spline}  + \frac{4}{u-x_5} + \frac{4}{x_6-u}\,.\]
Hence, $\|f^v_{y}(x) \|_{\spline} \leq \|f^v_{-y}(x)\|_{\spline}$ and $f^{v}(x) = f^v_{y}(x) $. 
Hence, for a pair of identically labeled points $x_5$ and $x_6$ and all $x_5\leq v\leq x_6$, we have 
\begin{align}
\label{eq:score1Susimilar}
\score_{\spline} (v)=\|f^v(x)\|_{\spline}=\|f(x)\|_{\spline}
\end{align}
 which is a constant independent of $v$. 
\item Equations~\eqref{eq:score1Su} and~\eqref{eq:score1Susimilar}, show that in this setup
 \[\score_{\spline} (v) = \|f(x)\|_{\spline} \leq
 \score^{(1)}_{\spline} (u) \,.\]
\end{enumerate}
\subsection{Proof of Theorem~\ref{thm:DataSearchLinearSpline}}
\label{s:proofNNDB}

 Without loss of generality, assume $x_1=0$ and $x_2 = 1$. Then, if $y_1=y_2$, the statement of theorem is trivial. 

 Hnce,
 \begin{align*}
 f(x) =   \frac{x_2+x_1-2x}{x_2-x_1}\,, \quad \text{ for all } x_1<x<x_2\,.
 \end{align*}
 For $x_1<u<x_2$, we have $\es(u)=+1$ if $x_1<u< (x_1+x_2)/2$ and $\es(u)=-1$ if $(x_1+x_2)/2<u<x_2$. 
 First, we look at $x_1<u< (x_1+x_2)/2$:
 \begin{align*}
 f^u(x) =
 \begin{cases}
  1\,, \quad &\text{ for all } x_1<x<u
  \\
  \frac{x_2 +u - 2x}{x_2-u}
   \,, \quad &\text{ for all } u<x<x_2\,.
  \end{cases}
 \end{align*}
 Hence, for  $x_1<u< (x_1+x_2)/2$ 
 \begin{align*}
 \|f-f^u\|_{\data} 
 & = 
 \int_{x_1}^{x_2} \big[f^u(x)-f(x)\big]^2\mathrm{d}x
 \\
 & =
 \frac{4}{(x_2-x_1)^2}\int_{x_1}^{u}  (x-x_1)^2\mathrm{d}x + 
 \frac{4(u - x_1)^2}{(x_2-x_1)^2 (x_2 - u)^2 }\int_{u}^{x_2} (x_2-x)^2 \mathrm{d}x
 \end{align*}

 \begin{align*}
 \frac{\mathrm{d}}{\mathrm{d} u }\|f- f^u\|_{\data} 
 & = 
 \frac{4}{(x_2-x_1)^2}(u-x_1)^2
 -\frac{4(u - x_1)^2}{(x_2-x_1)^2 (x_2 - u)^2 }(x_2-u)^2
 \\ &\quad
 +
 \frac{8(u - x_1)}{3(x_2-x_1) (x_2 - u)^3 } (x_2 - u)^3
 \\
 &=
 \frac{8}{3} \frac{u - x_1}{x_2-x_1 }  \geq 0\,.
 \end{align*}
 Similarly, we can show that for $\frac{x_1+x_2}{2}\leq u \leq x_2$, $f^u(x) = f^u_{-}(x)$ and 
 \begin{align*}
 \frac{\mathrm{d}}{\mathrm{d} u }\|f-f^u\|_{\data} 
 \leq 0\,.
 \end{align*}


%% file: proofs.tex

\section{Maximin kernel based active learning in one dimension}
\subsection{Minimum norm interpolating function with Laplace Kernel in one dimension}
\label{s:AppLaplace}
We want to find the minimum norm interpolating function based on set of labeled samples $\Lc=\{(x_1,y_1),\cdots,(x_n,y_n)\}$ such that $x_1<x_2<\cdots,x_n$.

	First, let us look at the Kernel matrix for three neighboring points $x_1<x_2<x_3$ according to the Laplace Kernel.
\begin{align*}
K
&=
\begin{bmatrix}
1 &    e^{-(x_2-x_1)/h}    &    e^{-(x_3-x_1)/h}\\
e^{-(x_2-x_1)/h}   &   1   &     e^{-(x_3-x_2)/h}  \\
 e^{-(x_3-x_1)/h} &   e^{-(x_3-x_2)/h}  &   1  
\end{bmatrix}\,.
\end{align*} 
We define $d_1= e^{-(x_2-x_1)/h} $ and $d_2=e^{-(x_3-x_2)/h} $.
It can be shown that with the above structure
\begin{align*}
K^{-1}
&= 
\begin{bmatrix}
\frac{1}{1-d_1^2}&   \frac{-d_1}{1-d_1^2}   &  0\\ 
\frac{-d_1}{1-d_1^2}   &  \frac{1}{1-d_1^2}+\frac{1}{1-d_2^2} -1   &    \frac{-d_2}{1-d_2^2}  \\ 
0&  \frac{-d_2}{1-d_2^2}  &  \frac{1}{1-d_2^2}   \\
\end{bmatrix}\,.
\end{align*} 
In general, if we look into the Kernel matrix for the set of points $x_1<x_2<\cdots<x_n$, and define $d_{i} =e^{-(x_{i+1}-x_i)/h} $ for $1\leq i \leq n-1$. We define $d_0=d_n=0$. Using induction, on can show that the inverse of the Kernel matrix has a block diagonal form such that
\begin{align}
(K^{-1})_{1,1} &= \frac{1}{1-d_{1}^2}, 
\quad\quad\quad
(K^{-1})_{n,n} = \frac{1}{1-d_{n-1}^2}
\notag\\
(K^{-1})_{i,i} &= \frac{1}{1-d_{i-1}^2} +\frac{1}{1-d_{i}^2} - 1, \quad\quad \forall 1<i<n
\notag\\
(K^{-1})_{i,i+1} &=(K^{-1})_{i+1,i}= \frac{-d_{i}}{1-d_{i}^2}, \quad \quad \forall 1\leq i<n
\label{eq:LaplaceBlockD}
\end{align}
and the remaining elements of matrix of matrix $K^{-1}$ is zero, \textit{i.e.}, $(K^{-1})_{i,j}$ for $|j-i|\geq 2$.

Using~\eqref{eq:InterpFnc} and the above characterization of matrix $\Kb^{-1}$, we can show that the  the minimum norm interpolating function based on set of labeled samples $\Lc=\{(x_1,y_1),\cdots,(x_n,y_n)\}$ such that $x_1<x_2<\cdots,x_n$ has the following form:
\begin{align}
f(x) & = \frac{1}{1+y_1y_2\, d_1} \,y_1k(x_1,x) +\frac{1}{1+y_n y_{n_1}\, d_{n-1}}\, y_n k(x_n,x) 
\nonumber\\
&\quad+
\sum_{i=2}^{n-1} 
\Big[\frac{1}{1+ y_i y_{i-1} d_{i-1}} + \frac{1}{1+y_i y_{i+1} d_{i}} - 1\Big]\,
y_i  k(x_i,x)\,.
\label{eq:InterpLaplaceOneD}
\end{align}

\subsection{Criterion $\score_{\Hc}$ with Laplace Kernel in one dimension (Proposition~\ref{p:Laplace-1D})}
\begin{proof}[Proof of Proposition 1]
Looking at three neighboring points $x_1<x_2<\cdots<x_n$ labeled $\mathbf{y}=\{y_1,y_2,\cdots,y_n\}$, let $K$ be the kernel matrix corresponding to $\{x_1,x_2,\cdots,x_n\}$. Then, using  the block diagonal structure given in Equation~\eqref{eq:LaplaceBlockD}, we compute $\|f(x)\|_{\Hc}$:
\begin{align}
\|f(x)\|_{\Hc}
& = 
\mathbf{y}^T  K^{-1} \mathbf{y}
\notag\\& = 
\mathsf{trace}(K^{-1}) + \sum_{i=1}^{n-1}2 y_i y_{i+1} (K^{-1})_{i,i+1} 
\notag \\& = 
-(n-2)+2\sum_{i=1}^{n-1}\frac{1 -y_i y_{i+1} \exp(-(x_{i+1}-x_i)/h) }
{1- \exp(-2(x_{i+1}-x_i)/h)} 
\notag \\& = 
-(n-2)+2\sum_{i=1}^{n-1}\frac{1}
{1+y_i y_{i+1} \exp(-(x_{i+1}-x_i)/h)} \,.
\label{eq:NormFLaplaceK}
\end{align}
This factorization of $\|f(x)\|_{\Hc}$ implies that if $x_i<u<x_{j+1}$, we have 
\begin{align}
\|f_{\es}^u(x)\|_{\Hc}
 = &
\|f(x)\|_{\Hc} -1 
- \frac{2}{1+y_j y_{j+1} \exp(-(x_{j+1}-x_j)/h)} 
\notag\\&
+\frac{2}{1+\es y_j \exp(-(u-x_j)/h)} 
+ \frac{2}{1+\es y_{j+1} \exp(-(x_{j+1}-u)/h)} \,.
\label{eq:LapNormNewInterp}
\end{align}

This factorization of $\|f_{\es}^u(x)\|_{\Hc}$ implies that given $n$ labeled points, to find the global maximizer of $\score_{\Hc}(u)$, the following strategy  works: 

\textbf{Step~1:} For each $j$,  we find the maximizer of  $\score_{\Hc}(u)$ for the unlabeled points $u$ such that $x_j<u<x_{j+1}$. We define
\begin{align}
\label{eq:Laplaceustarj}
u_j^*=\arg\max_{x_j<u<x_{j+1}}\score_{\Hc}(u)
\end{align}
The above factorization~\eqref{eq:LapNormNewInterp} shows that to find $u_j^*$, we can alternatively find the maximizer of the sore given a configuration of only two labeled points $\Lc'=\{(x_i,y_j), (x_{j+1},y_{j
+1})\}$ and unlabeled points $x_{j}< u< x_{j+1}$ to find the maximizer of the last two terms of the above factorization. The maximzer of $\score_{\Hc}(u)$ for $u$ such that $x_j<u<x_{j+1}$ given $\Lc'$ is the same as the maximizer of $\score_{\Hc}(u)$ for $u$ such that $x_j<u<x_{j+1}$ given $\Lc$.

\textbf{Step~2:} 
Note that $\arg\max_{u}\score_{\Hc}(u) = u_{j^*}^*$ such that $j^*=\arg\max_{j}\score_{\Hc}(u_j^*)$. 
We can use~\eqref{eq:LapNormNewInterp} to compare  $\score_{\Hc}(u_{j}^*)$ for various values of $j$. 

In what follows, we show that using Laplace Kernel, whether $y_j=y_{j+1}$ or $y_j\neq y_{j+1}$, we always have $u_j^* = (x_j+x_{j+1})/2$ and $\es(u_j^* )=y_j$. Hence, using~\eqref{eq:LapNormNewInterp}, we have 
\begin{align}
\score_{\Hc}(u_{j}^*) =& 
\|f(x)\|_{\Hc} -1 
- \frac{2}{1+y_j y_{j+1} \exp(-(x_{j+1}-x_j)/h)} 
\notag\\&
+\frac{2}{1+ \exp(-(x_{j+1}-x_j)/2h)} 
+ \frac{2}{1+ y_j y_{j+1} \exp(-(x_{j+1}-x_j)/2h)} \,.
\label{eq:scoreujLaplace}
\end{align}

\textbf{Step~1:}
\begin{enumerate}
\item 
Consider two neighboring points $x_1<x_3$ labeled $\mathbf{y}=\{-,+\}$. Let $u$ be such that $x_1<u<x_3$ and $t=y_3$. Define $A= e^{-(u-x_1)/h} $ and $B=e^{-(x_3-u)/h} $. Define the matrix $K$ to be the kernel matrix corresponding to the points $x_1,u,x_3$. Then using~\eqref{eq:NormFLaplaceK}
\begin{align*}
\|f^u_{\es}(x)\|^2 
& = 
\frac{2}{1-\es A} + \frac{2}{1+\es B}  -1
\end{align*}
Hence, 
\begin{align*}
\es(u) &= \argmin_{\es} \|f^u_{\es}(x)\|^2  =
\begin{cases}
 +1, \quad \quad \text{if }  u>\frac{x_1+x_3}{2} \text{ or } B>A
 \\
 -1, \quad \quad \text{if } u\leq \frac{x_1+x_3}{2} \text{ or } B\leq A
 \end{cases}
\end{align*}
Without loss of generality,  assume $A\geq B$ or $ u\leq \frac{x_1+x_3}{2} $, then the block diagonal structure given in Equation~\eqref{eq:LaplaceBlockD} implies
 \begin{align*}
\score_{\Hc}(u) 
 &=2 \left\{  \frac{1}{1+A} + \frac{1}{1-B} \right\} -1
\end{align*}
 Note that for all $u$ such that $x_1<u<x_3$, we have $AB =e^{-(x_3-x_1)/h}=C $ and is constant. So 
 \begin{align*}
\max_{u\leq \frac{x_1+x_3}{2} } \score_{\Hc}(u) 
&=
\max_{A,B, \text{ s.t. } AB=C,  A\geq B } 2 \left\{  \frac{1}{1+A} + \frac{1}{1-B} \right\} -1
\\ & = 
\frac{4}{1-C} -1 = \frac{4}{1-e^{-(x_3-x_1)/h}} -1 
 \end{align*}
 which is attained when $A=B$ or $u = \frac{x_1+x_3}{2} $. 
 
 This gives the following statement:
 For neighboring labeled points $x_1<x_3$ such that $y_1\neq y_3$, we have 
\begin{align}
 \max_{s_1\leq u\leq x_3 } \score_{\Hc}(u)  = \score_{\Hc}( \frac{x_1+x_3}{2} )  =\frac{4}{1-e^{-(x_3-x_1)/h}}-1\,.
\label{eq:score1Prop1}
 \end{align}

 Note that the above function is decreasing in $x_3-x_1$. 
\item 
Consider two neighboring points $x_1<x_3$ labeled $\mathbf{y}=\{+,+\}$. Let $u$ such that $x_1<u<x_3$ and $t=y_3$. Define $A= e^{-(u-x_1)/h} $ and $B=e^{-(x_3-u)/h} $. 
 \begin{align*}
\max_{x_1<u<x_3} \score_{\Hc}(u) 
&=
\max_{A,B, \text{ s.t. } AB=C} 2 \left\{  \frac{1}{1+A} + \frac{1}{1+B} \right\} -1
\\
& = 
\frac{4}{1+\sqrt{C}} -1
= 
\frac{4}{1+{e^{-(x_3-x_1)/2h}}} -1
 \end{align*}
 
  This gives the following statement:
 For neighboring labeled points $x_1<x_3$ such that $y_1= y_3$, we have 
 \begin{align}
 \max_{s_1\leq u\leq x_3 } \score_{\Hc}(u)   = \score_{\Hc}( \frac{x_1+x_3}{2} )  =\frac{4}{1+{e^{-(x_3-x_1)/2h}}} -1 \,.
\label{eq:score1Prop2}
 \end{align}

 Note that the above function is increasing in $x_3-x_1$. 
\end{enumerate}

\textbf{Step~2:}
In Step~1, we showed that the maximizer of score at each interval between two neighboring points is achieved in the center of the interval, i.e., $u^*_{j}=\frac{x_j+x_{j+1}}{2}$  with notation defined in~\eqref{eq:Laplaceustarj}. Now to compare $ \score_{\Hc}(u^*_{j})$ for various $j$, we look at the following properties derived from the formulation in~\eqref{eq:scoreujLaplace}:

\begin{itemize}
\item If $y_j\neq y_{j+1}$, then 
$$\score_{\Hc}(u^*_{j}) = 
\|f(x)\|_{\Hc} -1 
+ \frac{2}{1- \exp(-(x_{j+1}-x_j)/h)} 
\geq 
\|f(x)\|_{\Hc} +1\,. 
$$
Note that if $y_j\neq y_{j+1}$, then  $\score_{\Hc}(u^*_{j}) $ is increasing in $x_{j+1}-x_j$.

\item If $y_j= y_{j+1}$, then 
$$\score_{\Hc}(u^*_{j}) = 
\|f(x)\|_{\Hc} -1 
- \frac{2}{1+ \exp(-(x_{j+1}-x_j)/h)} 
+ \frac{4}{1+ \exp(-(x_{j+1}-x_j)/2h)} 
\leq 
\|f(x)\|_{\Hc} +1\,. 
$$
Note that if $y_j= y_{j+1}$, then  $\score_{\Hc}(u^*_{j}) $ is decreasing in $x_{j+1}-x_j$.
\end{itemize}
The above two properties give the statement of the second part of proposition. 
\end{proof}

\subsection{Max Min criteria Binary Search (Corollary~\ref{cor:binsearch})}
\label{s:AppBinarySearch}

  According to the last property in Proposition~\ref{p:Laplace-1D} the
  first sample selected will be at the midpoint of the unit interval
  and the second point will be at $1/4$ or $3/4$.  If the labels
  agree, then the next sample will be at the midpoint of the largest
  subinterval (e.g., at $3/4$ if $1/4$ was sampled first). Sampling at
  the midpoints of the largest subinterval between a consecutive pairs
  labeled points continues until a point with the opposite label is
  found. Once a point the with opposite label have been found,
  Proposition~\ref{p:Laplace-1D} implies that subsequent samples
  repeatedly bisect the subinterval between the closest pair of
  oppositely labeled points. This bisection process will terminate
  with two neighboring points having opposite labels, thus identifying
  one boundary/threshold of $g$. The total number of labels collected
  by this bisection process is at most $\log_2 N$.  After this, the
  algorithm alternates between the two situations above. It performs
  bisection on the subinterval between the close pair of oppositely
  labeled points, if such an interval exists. If not, it samples at
  the midpoint of the largest subinterval between a consecutive pairs
  of labeled points.  The stated label complexity bound follows from
  the assumptions that there are $K$ thresholds and the length of each
  piece (subinterval between thresholds) is on the order of $1/K$.

%% file: Gaussian.tex
\subsection{One Dimensional Functions with Radial Basis Kernels (Proposition~\ref{p:RadialBasis})}
\label{s:AppRadialBasis}
\begin{proof}[Proof of Proposition \ref{p:RadialBasis} on maximum score with radial basis kernels]

For the ease of notation, for fixed $p$ and $h$, we define $a,b>0$ as the normalized distance between samples such that $x_2-x_1 = b h^{1/p}$ and
$x_3-x_2 = a h^{1/p}$. For $x_2<u<x_3$, we define  $0<c<a$ such that the distance between the point $u$ and $x_2$ is $u-x_2 = c h^{1/p}$, as in Figure~\ref{f:Gaussian3pt}. 
The proposition is based on the assumption that for any pair of  points, $|x-x'|\geq \Delta$ and $h$ is sufficiently small that  $a,b,c,a-c \geq \Delta h^{-1/p}$. 

We want to show that the max score happens at the zero crossing of function $f(u)$. Since we normalized all pairwise distances by $h^{1/p}$, instead we will show that there exists a constant $D$ such that if $a,b,c,a-c \geq D$, then the max score is achieved at the zero crossing. 

\begin{figure}[h]
\begin{centering}
\includegraphics[width=.8\linewidth]{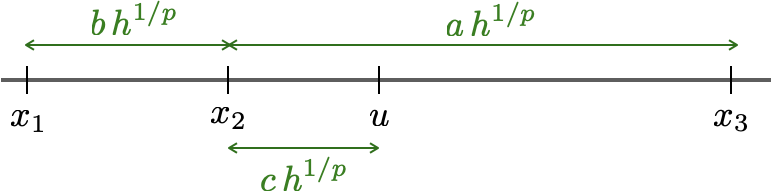}
\end{centering}
\caption{The samples $x_1,x_2$ and $x_3$ are the labeled samples such that $y_1=y_2= +1$ and $y_3=-1$.}
\label{f:Gaussian3pt}
\end{figure}
Note that $\|f^{(u)}_{+}\|$ depends on the location of point $u$, characterized by the normalized distance between $u$ and $x_2$ denoted by  $c$.
We want to prove that for small enough bandwidth,  $\|f^{(u)}_{+}\|$ is increasing in $c$ for $\Delta h^{-1/p}<c<a-\Delta h^{-1/p}$. We can use similar argument to show that $\|f^{(u)}_{-}\|$ is decreasing in $c$. This implies
$$\max_{\Delta h^{-1/p}<c<a-\Delta h^{-1/p}}\,\,\,\, \min_{\es\in\{\-,+\} } \|f^{(u)}_{\es}\| = \|f^{(u^*)}_{+}\| = \|f^{(u^*)}_{-}\|$$
with $u^*$ defined to be the point in which $f(u^*)=0$.

 To do so, we will show that $\frac{\mathrm{d}}{\mathrm{d} c}\|f^{(u)}_{+}\|>0$ in the interval $\Delta h^{-1/p}<c<a-\Delta h^{-1/p}$. 

In the proof of Proposition \ref{p:labelKernel}, we showed the following form for the function  $\|f^{(u)}_{+}\|$,
\[\|f^u_+(x)\| = \frac{\big[1-\, f(u)\big]^2}{ 1-\abv_u^T  \, \Kb^{-1} \abv_u}\]
where $\Kb$ is the kernel matrix for the points $x_1,x_2$ and $x_3$.
 The vector $\abv_v$ is defined to be $\abv_v = [K(x_1,u), K(x_2,u),K(x_3,u)]^T$. 
 The term $f(u)$ is the minimum norm interpolating function based on the points $x_1,x_2$ and $x_3$ and their labels  evaluated at $u$. Equation~\ref{eq:InterpFnc} shows that 
 \[f(u) = \yb^T \Kb^{-1} \abv_u\,.\]

First, we look at $\Kb$ and its inverse in the setup explained above. Using the definition of Radial basis kernels in Equation (4),

\begin{align}
	\mathbf{K} = 
	\begin{bmatrix}
	1 & \exp(-b^p) & \exp[-(a+b)^p] \\
	 \exp(-b^p) & 1 & \exp(-a^p)\\
	\exp[-(a+b)^p] & \exp(-a^p) & 1
	\end{bmatrix}
\end{align}

Hence, 

\begin{align*}
	\mathbf{K}^{-1} 
	&=
	 \frac{1}{|\mathbf{K}|}
	\begin{bmatrix}
	1-e^{-2a^p} && -e^{-b^p}  + e^{-a^p - (b+a)^p} && e^{-a^p-b^p} -e^{-(a+b)^p} 
	\\\\
 -e^{-b^p}  + e^{-a^p - (b+a)^p}
 && 	1-e^{-2(a+b)^p}  && -e^{-a^p} + e^{-b^p - (a+b)^p}\\\\
e^{-a^p-b^p} -e^{-(a+b)^p}  && 
-e^{-a^p} + e^{-b^p - (a+b)^p}
&& 
	1-e^{-2b^p} 
	\end{bmatrix}
	\\\\&=
	 \frac{1}{|\mathbf{K}|}
	 \begin{bmatrix}
	1-e^{-2a^p} && 
	-e^{-b^p}[1-\epsilon_1]
	 && e^{-a^p-b^p} [1-\epsilon_2]
	\\\\
	-e^{-b^p}[1-\epsilon_1]
 && 	1-e^{-2(a+b)^p}  && 
 -e^{-a^p} [1-\epsilon_3]
 \\\\
 e^{-a^p-b^p} [1-\epsilon_2]  && 
-e^{-a^p} [1-\epsilon_3]
&& 
	1-e^{-2b^p} 
	\end{bmatrix}
\end{align*}

The determinant of matrix $\Kb$ is
\begin{align*}
	|\mathbf{K}|
& = 
1 + 2\exp[-a^p - b^p -(a+b)^p]
-  \exp[-2(a+b)^p]
- \exp(-2 a^p) - \exp(-2 b^p)
& = 1-\epsilon'
\end{align*}
where we defined
\begin{align*}
\epsilon'  & = 
  \exp[-2(a+b)^p]
+ \exp(-2 a^p) + \exp(-2 b^p)
-2\exp[-a^p - b^p -(a+b)^p]
\\
\epsilon_1 & = 
\exp\big({b^p-a^p - (b+a)^p}\big)
\\
\epsilon_2 & = 
\exp\big({b^p+a^p - (b+a)^p}\big)
\\
\epsilon_3 & = 
\exp\big({a^p-b^p - (b+a)^p}\big)
\end{align*}
Note that since for $p\geq 1$, we have $(a+b)^p\geq a^p + b^p$, then $\epsilon',\epsilon_1,\epsilon_2,\epsilon_3\geq 0$. Also, there exists a constant $D$ such that if $a,b >D$, then $|\Kb|\geq 0.9$. 

The vector $\abv_v$ is 
\begin{align*}
\abv_v 
&= [K(x_1,u), K(x_2,u),K(x_3,u)]^T
 = 	
 \begin{bmatrix}
		\exp[-(c+b)^p] &
		\exp(-c^p) &
		\exp[-(a-c)^p]
	\end{bmatrix}^T
\end{align*}
Next, we compute $f(u)$
\begin{align*}
	f(u) & = 
	\mathbf{y}^T \mathbf{K}^{-1} \abv_u
	\\
	&= 
\frac{1}{|\mathbf{K}|} 
 \begin{bmatrix}
	1 & 1 & -1
\end{bmatrix}
 \begin{bmatrix}
	1-e^{-2a^p} && 
	-e^{-b^p}[1-\epsilon_1]
	 && e^{-a^p-b^p} [1-\epsilon_2]
	\\\\
	-e^{-b^p}[1-\epsilon_1]
 && 	1-e^{-2(a+b)^p}  && 
 -e^{-a^p} [1-\epsilon_3]
 \\\\
 e^{-a^p-b^p} [1-\epsilon_2]  && 
-e^{-a^p} [1-\epsilon_3]
&& 
	1-e^{-2b^p} 
	\end{bmatrix}
	\begin{bmatrix}
		e^{-(c+b)^p}\\\\
		e^{-c^p}\\\\
		e^{-(a-c)^p}
	\end{bmatrix}
		\\ \\
		& =
		\frac{1}{|\mathbf{K}|} 
	\left\{		e^{-(c+b)^p} [1-\epsilon_4]
+		e^{-c^p} [1-\epsilon_5] 
-		e^{-(a-c)^p}  [1-\epsilon_6] 
\right\}
\end{align*}
where we defined
\begin{align*}
		\epsilon_4 &=  e^{-2a^p} + e^{-b^p}[1-\epsilon_1]  +  e^{-a^p-b^p} [1-\epsilon_2] \\
		\epsilon_ 5 &= e^{-b^p}[1-\epsilon_1] +e^{-2(a+b)^p} -e^{-a^p} [1-\epsilon_3] \\
		\epsilon_6 & = e^{-a^p-b^p} [1-\epsilon_2] -e^{-a^p} [1-\epsilon_3] +e^{-2b^p} 
\end{align*}
So there exists a constant $D$ such that if $a,b,c,a-c >D$, then $|f(u)|\leq 1$ and $\epsilon_4, \epsilon_5, \epsilon_6 \leq 0.1$. 

Next, we derive the derivative of $f(u)$ as  a function of $c$
\begin{align*}
|\Kb| \frac{\mathrm{d}}{\mathrm{d} c}	f(u) 
& = 	-p(c+b)^{p-1} e^{-(c+b)^p} [1-\epsilon_4]
-p \,.c^{p-1}	e^{-c^p} [1-\epsilon_5] 
-p(a-c)^{p-1}		e^{-(a-c)^p}  [1-\epsilon_6]\,.
\\
&\leq
-0.9\big\{
p \,.c^{p-1}	e^{-c^p} +
p(a-c)^{p-1}		e^{-(a-c)^p}  
\big\}
	\end{align*}
	where the last inequality uses $\epsilon_4, \epsilon_5, \epsilon_6 \leq 0.1$.
	
Next, we compute $\abv_u^T \Kb^{-1}\abv_u$	 and its derivative with respect to $c$.
To use  the formulation computed in the proof of Proposition 1, we compute  $	\abv_u^T \mathbf{K}^{-1} \abv_u$ to be
\begin{align*}
|\Kb| \abv_u^T \Kb^{-1}\abv_u
& = e^{-2(b+c)^p} \big[	1-e^{-2a^p}\big] 
		+  e^{-2 c^p} \big[	1-e^{-2(a+b)^p}\big]
		 + e^{-2(a-c)^p}\, \big[ 	1-e^{-2b^p} \big]
		 \\& \quad - 2 \exp{\big(- b^p - c^p - (b+c)^p \big)}\,\,[1-\epsilon_1] 
		 \\&\quad +2 \exp{\big( - b^p - a^p -(c+b)^p-(a-c)^p\big)}\,\,[1-\epsilon_2]
		 \\& \quad - 2 \exp{\big(- a^p - c^p - (a-c)^p\big)}\,\,[1-\epsilon_3]
		 \\
	1- \abv_u^T \Kb^{-1}\abv_u	 &\geq 0.9
\end{align*}
where the last inequality holds for large enough constant $D$ such that $a,b,c,a-c>D$.
Hence,

\begin{align*}
	|\Kb| \frac{\mathrm{d}}{\mathrm{d}c}  \abv_u^T \Kb^{-1}\abv_u
		 & = 
		-2p \,(b+c)^{p-1} e^{-2(b+c)^p} \big[	1-e^{-2a^p}\big] \\
		&\quad
		-  2p c^{p-1}e^{-2 c^p} \big[	1-e^{-2(a+b)^p}\big]
		 +2p (a-c)^{p-1} e^{-2(a-c)^p}\, \big[ 	1-e^{-2b^p} \big]
		 \\& \quad + 2p[ c^{p-1} + (b+c)^{p-1}]\,\,  \exp{\big(- b^p - c^p - (b+c)^p \big)}\,\,[1-\epsilon_1] 
		 \\&\quad - 2 p [(c+b)^{p-1}-(a-c)^{p-1}]\,\,  \exp{\big( - b^p - a^p -(c+b)^p-(a-c)^p\big)}\,\,[1-\epsilon_2]
		 \\& \quad + 2p [ c^{p-1} - (a-c)^{p-1}]\,\, \exp{\big(- a^p - c^p - (a-c)^p\big)}\,\,[1-\epsilon_3]
		 \\
		\big|  \frac{\mathrm{d}}{\mathrm{d}c}  \abv_u^T \Kb^{-1}\abv_u \big|
		  &\leq 
		 8p c^{p-1} e^{-2c^p} 		 +8p( a-c)^{p-1} e^{-2(a-c)^p}
\end{align*}
The remaining of the proof is based on the assumption that there exists a constant large enough value $D$ such that $a,b,c,a-c >D$. As we saw, this implies $f(u)< 1$. Plugging in the above computations in to the derivative of the function $\|f^u_+(x)\|$ with respect to $c$, after doing some algebra, we see that this function is increasing in $c$. We give the sketch of this algebra here:
\begin{align*}
 \frac{\mathrm{d}}{\mathrm{d} c}	\|f^u_+(x)\| 
& = 	
\big(1-f(u)\big) \frac{-2(1-\abv_u^T \Kb^{-1} \abv_u)\,  \frac{\mathrm{d}}{\mathrm{d} c}	f(u) + \big(1-f(u)\big)    \frac{\mathrm{d}}{\mathrm{d} c}	\abv_u^T \Kb^{-1} \abv_u}
{(1-\abv_u^T \Kb^{-1} \abv_u)^2}\\
& \geq
\big(1-f(u)\big) \frac{
1.5 \big\{
p \,.c^{p-1}	e^{-c^p} +
p(a-c)^{p-1}		e^{-(a-c)^p}  
\big\} 
+ \big(1-f(u)\big)    \frac{\mathrm{d}}{\mathrm{d} c}	\abv_u^T \Kb^{-1} \abv_u}
{(1-\abv_u^T \Kb^{-1} \abv_u)^2}
\\
&\geq 0\,.
\end{align*}
The last inequality uses the bounds proved above.

 Similarly, we can prove that $\|f^u_-(x)\|$ is decreasing in $c$. this implies that the $\score$ function is maximize in the zero crossing of function $f(x)$.
 \end{proof}

%% file: proofsRKHSDB.tex
\section{Maximin kernel based active learning with clustered data}
To prove the statement of theorems presented in Section~\ref{s:result}, we introduce some notations consistent with the notation introduced in Section~\ref{sec:RKHS}.
 Given a set of labeled samples $\Lc=\{(x_1,y_1),\cdots,(x_L,y_L)\}$, define the $L$ by $L$ matrix $\Kb=[k(x_i,x_j)]_{1\leq i,j\leq L}$ and vector $\yb=[y_1,\cdots,y_L]^T$. 

Recall that $\Uc$ is a set of unlabled examples. For $u\in\Uc$ and 
$\es\in\{-1,+1\}$, let $\abv_u=[k(x_1,u), \cdots,k(x_L,u)]^T$ and 
$\Kt_u$ be the $L+1$ by $L+1$ matrix such that
\[
\Kt_u = 
\begin{bmatrix}
\Kb& \abv_u\\
\abv_u^T & 1
\end{bmatrix}
\quad
\text{and}
\quad
\yt_{\es}=
\begin{bmatrix}
\yb
\\
t
\end{bmatrix}\,.
\]
Let $\Bc_{d,p}(r;c)$ be the $d$ dimensional $\ell_p$ ball with radius $r$ centered at $c$ (defined in~\eqref{eq:ballDef}).
Let $V_{d,p}(r)$ be the volume of $\Bc_{d,p}(r;0)$ with respect to the
Lebesgue measure.
\subsection{Proof of Theorem~\ref{t:FirstPoint}}
\label{s:ProofFirstPoint}
The statement of theorem implies that when the data is clustered and
distributed uniformly in $\ell_p$ balls, with centers far enough from
each other, the first selected point using the $\score_{\data}$ function
defined in~\eqref{dscore} is in the largest ball.  To prove
this, we will show that the $\score_{\data}(c_1)$, as defined
in~\eqref{dscore} is larger than $\score_{\data}(v)$ for any
$v\notin B_1$ where $c_1$ is the center of $B_1$.  Note that this does
not imply that the first selected point coincides with the center of
$B_1$. It guarantees that the largest ball contains at least one point
with a score larger than that of every point in other balls.  

Since  $\Lc=\varnothing$, the empty set, the current interpolating function is uniformly zero everywhere $f(x)=0$ (according to the definition~\eqref{eq:InterpFnc}). 
According to the Equations~\eqref{eq:defFtu} and~\eqref{eq:t1def}, for all $u\in\Uc$, we can choose $\es(u)$ to be equal to $+1$ or $-1$. 
We choose $\es(u)=+1$  without loss of generality for all $u\in\Uc$.

Using~\eqref{eq:InterpFnc}, adding any point $u\in\Uc$ with label $\es(u)$ to $\Lc$ would give the new interpolating function
\[f^u(x):= f^u_{\es(u)}(x) = k(u,x) = \exp\big( -\frac{1}{h} \|x-u\|_p \big)\,.\]
Hence, since $\mathbb{P}_X(x)$ is uniform over $\Xc=\cup_{i=1}^M  B_i$ 
\begin{align*}
\score_{\data}(u) 
&= 
\int_{x\in\Xc}  \exp\big( -\frac{2}{h} \|x-u\|_p \big) \mathrm{d} P_X(x)
 = \frac{1}{V }
\sum_{i=1}^M 
\int_{x\in B_i}  \exp\big( -\frac{2}{h} \|x-u\|_p \big) \mathrm{d} x
\end{align*}
where we defined $V=\sum_{i=1}^M V_{d,p}(r_i)$ to be the total volume of $\mathcal{X}$. 
So, to compute $\score_{\data}(c_1)$,

\begin{align}
V \score_{\data}(c_1) 
& =
\sum_{i=1}^M 
\int_{x\in B_i}  \exp\big( -\frac{2}{h} \|x-c_1\|_p \big) \,\mathrm{d} x
\geq 
\int_{x\in B_1}  \exp\big( -\frac{2}{h} \|x-c_1\|_p \big)\, \mathrm{d} x
\nonumber
\\
&=
\int_{s=0}^{r_1}  \exp\big( -\frac{2s}{h}  \big) \,\mathrm{d} V_{d,p}(s)
\label{eq:T1Pterm1}
\end{align}
where we used the change of variable $s= \|x-c_1\|_p$ in the last line. 

For $v\notin B_1$, we want to show that $\score_{\data}(v)\leq \score_{\data}(c_1)$. 
Let $v\in B_j$ such that $j\neq 1$.
\begin{align}
V\score_{\data}(v) 
&=
\int_{x\in B_j}  \exp\big( -\frac{2}{h} \|x-v\|_p \big) \mathrm{d} x
 +
\sum_{i=1, i\neq j}^M 
\int_{x\in B_i}  \exp\big( -\frac{2}{h} \|x-v\|_p \big) \mathrm{d} x
\label{eq:T1Pterm2}
\end{align}
We will bound each of above terms separately. 

For any $i\neq j$ and $x\in B_i$ application of triangle inequality gives 
$$\|x-v\|_p \geq \|c_i - c_j\| - \|x-c_i\| - \|v-c_j\|\geq  D$$ 
since $v\in B_j, x\in B_i$ and $ \|c_i - c_j \|_p\geq D+2 r_1$, $\|x -c_i \|_p \leq r_i\leq r_1$ and $\|v-c_j\|_p \leq r_j\leq r_1$.  Hence, 
\begin{align}
\sum_{i=1, i\neq j}^M &
\int_{x\in B_i}  \exp\big( -\frac{2}{h} \|x-v\|_p \big) \mathrm{d} x
\leq 
\sum_{i=1, i\neq j}^M
\int_{x\in B_i}  e^{ -2D/h} \mathrm{d} x
\leq
 e^{ -2D/h} 
\sum_{i=1, i\neq j}^M V_i\,.
\label{eq:T1Pterm3}
\end{align}

Lemma~\ref{l:CenterAlignment} shows that the first term
in~\eqref{eq:T1Pterm2} is largest when $v$ coincides with
$c_j$. 
Hence,
\begin{align}
\int_{x\in B_j} 
\hspace{-.1in}
\exp\big( & -\frac{2}{h} \|x-v\|_p \big) \mathrm{d} x
\leq 
\int_{x\in B_j} 
\hspace{-.1in} \exp\big( -\frac{2}{h} \|x-c_j\|_p \big) \mathrm{d} x
=
\int_{s=0}^{r_j}   e^{ -2s/h} \, \,\mathrm{d} V_{d,p}(s)\,.
\label{eq:T1Pterm4}
\end{align}
Equations~\eqref{eq:T1Pterm1},~\eqref{eq:T1Pterm2},~\eqref{eq:T1Pterm3}, and~\eqref{eq:T1Pterm4} give
\begin{align*}
V  \score_{\data}(c_1) - V\score_{\data}(v)
&\geq 
\int_{s=r_j}^{r_1}  \exp\big( -\frac{2s}{h}  \big)\, \,\mathrm{d} V_{d,p}(s)
-
\exp\big( -\frac{2}{h} D \big)
\sum_{i=1, i\neq j}^M V_i
\\
&\geq 
 \exp\big( -\frac{2r_1}{h}  \big)
\,\,\big[ V_1 -  V_j\big]
 -
M \,V_1 \exp\big( -\frac{2D}{h}  \big)\,.
\end{align*}

Hence, 
\begin{align*}
\frac{V}{V_1} \Big[ \score_{\data}(c_1) - \score_{\data}(v)\Big]
&\geq 
 \exp\big( -\frac{2r_1}{h}  \big)
\,\,\big[ 1-\frac{V_j}{V_1}\big]
 -
M \, \exp\big(-2\frac{D}{h}  \big)
\\
&
\stackrel{(a)}{\geq }
 \exp\big( -\frac{2r_1}{h}  \big)
\,\,\big[ 1-\big(\frac{r_2}{r_1}\big)^d\big]
 -
M \, \exp\big( -2\frac{D}{h}  \big)
\geq 0\,,
\end{align*}
where  inequality (a) is due to the property that $$V_{d,p}(r) = \frac{\big[ 2r\,\Gamma(1+1/p)  \big]^d}{\Gamma(1+d/p)}$$
and $r_j\leq r_2$ for all $j\neq 1$.
Also, the assumption 
$$D>  \frac{h}{2}\Big[\ln M - \ln
\big(1-(r_2/r_1)^d\big)\Big],$$ 
and $r_1\leq h/2$
made in the statement of the theorem, 
yields the last inequality.

\begin{lemma}
\label{l:CenterAlignment}
For any $v\in B=\Bc_{d,p}(r,c)$, we have 
\[
\int_{x\in B} \exp\Big(\frac{-\|x-v\|}{h} \Big)\, \mathrm{d}x
\leq 
\int_{x\in B} \exp\Big(\frac{-\|x-c\|}{h} \Big)\, \mathrm{d}x\,.
 \]
\end{lemma}
\begin{proof}
To prove the statement of lemma, instead of looking at the integration of two different functions 
$ \exp\Big(\frac{-\|x-v\|}{h} \Big)$ and $ \exp\Big(\frac{-\|x-c\|}{h} \Big)$ on one ball, we look at two balls each centered on $v$  and $c$. This intermediate steps helps up prove the statement of lemma.

Let $B_1=\Bc_{d,p}(r,c)$ and $B_2=\Bc_{d,p}(r,v)$. Then, 
\[
\int_{x\in B_1} \exp\Big(\frac{-\|x-c\|}{h} \Big)\, \mathrm{d}x
=
\int_{x\in B_2} \exp\Big(\frac{-\|x-v\|}{h} \Big)\, \mathrm{d}x\,.
\]
Hence, to prove the statement of the lemma, we upper bound the following:
\begin{align*}
 \int_{x\in B_1}
\hspace{-.1in}
 \exp\Big(\frac{-\|x-v\|}{h} \Big)\, \mathrm{d}x
&-
\int_{x\in B_1}
\hspace{-.1in}
 \exp\Big(\frac{-\|x-c\|}{h} \Big)\, \mathrm{d}x
\\
&=
\int_{x\in B_1} 
\hspace{-.1in}\exp\Big(\frac{-\|x-v\|}{h} \Big)\, \mathrm{d}x
-
\int_{x\in B_2} \hspace{-.1in}
\exp\Big(\frac{-\|x-v\|}{h} \Big)\, \mathrm{d}x
\\
&=
\int_{x\in B_1\setminus B_2}
\hspace{-.2in} \exp\Big(\frac{-\|x-v\|}{h} \Big)\, \mathrm{d}x
-
\int_{x\in B_2 \setminus B_1 }
\hspace{-.2in}
 \exp\Big(\frac{-\|x-v\|}{h} \Big)\, \mathrm{d}x
\\
&\stackrel{(a)}{\leq }
e^{-r/h}\int_{x\in B_1\setminus B_2} \hspace{-.2in}\mathrm{d}x
-
e^{-r/h} \int_{x\in B_2 \setminus B_1 } \hspace{-.2in}\mathrm{d}x
\stackrel{(b)}=
0\,.
\end{align*}
Inequality (a) is due to the fact that since $v$ is the center of $B_2$, for $x\in B_1\setminus B_2$, $\|x-v\|\geq r$ and for $x\in B_2\setminus B_1$, $\|x-v\|\leq r$. Equality (b) is due to the fact that volume of $B_2 \setminus B_1$ is equal to the volume of $B_1 \setminus B_2$.
\end{proof}

\subsection{Proof of Theorem~\ref{t:clusterExplore}}
\label{s:prooftclusterExplore}
The statement of theorem shows that if the data is clustered, and few
of the clusters has been labeled so far, the algorithm selects a
sample from a cluster which has not been labeled so far. 
To do so, without loss of generality, we  show that for any  $u\in  B_L$, and there exists a $v\in B_{L+1}$ such that 
$\score_{\data}(v) > \score_{\data}(u)$.
The same argument shows that for any $i\leq L$ and any $u\in B_i$,  there exists a $v\in B_{L+1}$ such that  $\score_{\data}(v) > \score_{\data}(u)$. This proves that the score of any point in the labeled balls so far
is smaller than at least one point in the unlabeled clusters and hence the next point to be selected is in one of currently unlabeled balls. 

We will show that for any  $u\in  B_L$, and there exists a $v\in B_{L+1}$ such that 
$\score_{\data}(v) > \score_{\data}(u)$.
In particular, for any fixed $u\in B_1,$
 we choose
\begin{align}
\label{eq:defvclusters}
v=c_{L+1} + (u-c_1)\,.
\end{align}
We break the rest of the proof into five steps.


\noindent{\bf Step 1:}
First, we will look into the interpolator function $f(x)$ such that $f(x_i)=y_i$ for $(x_i,y_i)\in\Lc$, defined in~\eqref{eq:InterpFnc}.

Since $x_i\in B_i$ for $i=1,\cdots,L$, and $\|c_i-c_j\|_p > D+2r$, we have $\|x_i-x_j\|_p\geq D$ and $k(x_i,x_j)\leq e^{-D/h}$. Hence, matrix $\Kb$ can be decomposed as
\begin{align*}
\Kb= \mathbf{I}_L + e^{-D/h} \mathbf{E} 
\end{align*}
where $ \mathbf{I}_L$ is the identity $L\times L$ matrix and  matrix
$\mathbf{E}=[E_{i,j}]_{1\leq i,j\leq L}$  satisfies $0\leq E_{i,j}\leq 1$. 
Hence,  using Taylor series,
\begin{align}
\Kb^{-1}
&=
 \mathbf{I}_L + \sum_{n=1}^{\infty} (-1)^n e^{-nD/h} \mathbf{E}^n
\stackrel{(a)}{=}
\mathbf{I}_L + \widetilde{ \mathbf{E} }^{(1)}\sum_{n=1}^{\infty} e^{-nD/h} L^{n-1} 
\nonumber
\\
&\stackrel{(b)}{=}
\mathbf{I}_L + \frac{ e^{-D/h}}{1-L e^{-D/h}} \widetilde{ \mathbf{E} }^{(1)} 
\stackrel{(c)}{=}
\mathbf{I}_L + 2\,e^{-D/h} \widetilde{ \mathbf{E}}^{(2)} 
\label{eq:InverseFarK}
\end{align}
The matrices $ \widetilde{ \mathbf{E} }^{(1)}=[\widetilde{E}_{i,j}]_{1\leq i,j\leq L}$ and
 $ \widetilde{ \mathbf{E} }^{(2)}$
  also satisfy $|\widetilde{E}^{(1)}|_{i,j}\leq 1$ and $|\widetilde{E}^{(2)}|_{i,j}\leq 1$. 
   For any $n\geq 1$, the matrix $\mathbf{E}^n$ has elements smaller
   than $L^{n-1}$ (This can be proved using induction over $n$). This gives (a).
(b) is the summation of a geometric series (which holds since $D>h\log L$). 
(c) is due to the assumption $D> h\ln(2L)$.
Plugging this into~\eqref{eq:InterpFnc}  gives
\begin{align*}
f(x)&=\sum_{i=1}^L( y_i +\epsilon^{(f)} \gamma_i)k(x_i,x)
\end{align*}
where $\epsilon^{(f)} = 2L \,e^{-D/h}$. 
To make the notation easier, from now on, we will use the variables $\gamma_i$   with possibly different values in each line.  Note that the values of $\gamma_i$ depend on the elements of matrix $\widetilde{E}^{(2)}$ and realization of $y_i$ for $i=1,\cdots,L$.  But we  always have $|\gamma_i|\leq 1$.

\noindent{\bf Step 2:}
For any $v\in B_{L+1}$, we have $\|v-x_i\|_p\geq D$ for all $i=1,\cdots,L$. Hence,
 the matrix $\Kt_v$ defined in~\eqref{eq:defKtYt} takes the form
\begin{align*}
\Kt_v
& = 
 \mathbf{I}_{L+1} + e^{-D/h} \mathbf{E} 
\end{align*}
where  matrix $\mathbf{E}=[E_{i,j}]_{1\leq i,j\leq L+1}$  satisfies $|E_{i,j}|\leq 1$. 
Similar analysis as in step 1 and~\eqref{eq:InverseFarK} shows that for $v\in B_{L+1}$ and  any $\es\in\{-1,+1\}$,
(using definition of $f_v^{\es(v)}(x)$ in~\eqref{eq:defFtu})
we have
\begin{align*}
f_v^{\es}(x) & = \sum_{i=1}^L[ y_i + \epsilon^{(v)} \gamma_i]\,k(x_i,x) + [\es+\epsilon^{(v)} \gamma_{L+1}]\, k(v,x)
\end{align*}
where $\epsilon^{(v)}= 2(L+1) \,e^{-D/h}$. 
Hence, 
\begin{align*}
f^v_{\es}(x) -& f(x)
=
 \es\, k(x,v)
+
\big(\epsilon^{(v)} + \epsilon^{(f)}\big) \Big[\sum_{i=1}^{L} \gamma_i k(x,x_i) +\gamma_{L+1} k(x,v)\Big] 
\end{align*}
Note that the value of the variables $\gamma_i$ above might be different from the previous lines, but there exists  parameters $\gamma_i$ that satisfy the above equality and $|\gamma_i|\leq 1$.

\noindent{\bf Step 3: }
 For any $u\in B_L$, we will show that, $y_L f(u)\geq 0$.  According to Proposition~\ref{p:labelKernel} in  Section~\ref{s:PropRKHS}, this proves that  $\es(u)=y_L$: our estimation of label of any sample in ball $B_L$ is $y_L$, the label of the only currently labeled sample in $B_L$.
\begin{align*}
y_L  f(u) 
& = 
y_1\sum_{i=1}^L ( y_i +  \epsilon^{(f)}\gamma_i )k(x_i,u)
\\
& =
(1+ \epsilon^{(f)} y_L\gamma_L) k(x_L,u) 
+ 
y_L\sum_{i=1}^{L-1} ( y_i +  \epsilon^{(f)}\gamma_i  )k(x_i,u)
\\
&
\stackrel{(a)}{\geq} (1-\epsilon^{(f)}) e^{-2 r/h} - L (1+\epsilon^{(f)} ) e^{-D/h}
\stackrel{(b)}{\geq} 
0\,,
\end{align*}
where (a) is due to the following facts: since $x_L\in B_L$ and $u\in B_L$,  
we have $\|x_L-u\|\leq 2r$ and $k(x_L,u)\geq e^{-2r/h}$. 
Also, since $u \in B_L$, for $i\leq L-1$ we have 
$\|x_i-u\|\geq D$ and $k(x_i,u)\leq e^{-D/h}$. 
The assumptions  $D> 12 h \log (2M)$, $L<M$ and the definition of $\epsilon^{(f)} = 2L
  e^{-D/h}$ give $\epsilon^{(f)} \leq 1/100$.    
  Then using the assumption $r<h/3$ gives (b).

\noindent{\bf Step 4:}
Fix $u\in B_L$ and define $d:=\|u-x_L\|\leq 2r$. 
Step 3 above proves $\es(u)=y_L$. 
 Lemma~\ref{l:inversKt} shows that there exist parameters $\gamma_i$ such that $|\gamma_i|\leq 1$ and 
 the interpolating function $f^u_{\es(u)}(x)$ defined  in~\eqref{eq:defFtu} takes the form
\begin{align*}
f^u_{\es(u)}(x) 
& =
 \Big[ \frac{y_L}{1+e^{-d/h}} +L\epsilon^{(u)}\gamma_{L+1} \Big]  \,k(x,u) 
 +
  \Big[ \frac{y_L}{1+e^{-d/h}} +L \epsilon^{(u)}\gamma_{1} \Big] \,k(x,x_L)
  \\
&\quad +
\sum_{i=1}^{L-1} (y_i + \epsilon^{(u)}\gamma_{i} ) K(x,x_i) 
\end{align*}
where $\epsilon^{(u)} =  4 L^3 e^{-D/h} $.
Hence, 
\begin{align*}
f^u_{\es(u)}(x) - &f(x)
=
 \frac{y_L}{1+e^{-d/h}}  k(x,u) -  \frac{y_L e^{-d/h}}{1+e^{-d/h}}  k(x,x_L)
  \\
&\quad + 
\big[ L \epsilon^{(u)} +  \epsilon^{(f)} \big]
\Big[\sum_{i=1}^{L} \gamma_i k(x,x_i) + \gamma_{L+1} k(x,u)\Big]
\end{align*}

\noindent{\bf Step 5:} 
Hence, using the fact that $k(x,x')\leq 1$, we get
\begin{align*}
|f^v_{\es}(x) &- f(x)|^2
- |f^u_{\es(u)}(x) - f(x)|^2
\\
&\geq 
k^2(x,v)
-2(L+1)^2 \big[ L \epsilon^{(u)} +\epsilon^{(v)}+ 2 \epsilon^{(f)} \big]
 -\frac{1}{(1+e^{-d/h})^2 } 
 \Big[ k(x,u) -  e^{-d/h} k(x,x_L) \Big]^2
\end{align*}

Since $P_{X}(x)$ is uniform over $\cup_{j=1}^M B_j$, we want to show that 
\begin{align}
\label{eq:Proof2Term1}
\sum_{j=1}^ M
\int_{x\in B_j}
|f^v_{\es}(x) - f(x)|^2
- |f^u_{\es(u)}(x) - f(x)|^2
\,\mathrm{d}x
\geq 0\,.
\end{align}
To do so, we will bound the above term by
\begin{align}
&\int_{x\in B_{L+1}}
\hspace{-.2in}
k^2(x,v)
\,\mathrm{d}x
-
2(L+1)^2 \big[ L \epsilon^{(u)} +\epsilon^{(v)}+ 2 \epsilon^{(f)} \big]
\sum_{i=1}^M V_i 
 \nonumber
 \\
&-
\sum_{j=1}^ M
\int_{x\in B_j}
\frac{ \Big[ k(x,u) -  e^{-d/h} k(x,x_L) \Big]^2 }{(1+e^{-d/h})^2 } 
 \,\mathrm{d}x
 \nonumber
 \\
&\geq 
\int_{x\in B_{L+1}}
\hspace{-.2in}
k^2(x,v)
\,\mathrm{d}x
-\int_{x\in B_1}
\frac{ \Big[ k(x,u) -  e^{-d/h} k(x,x_L) \Big]^2 }{(1+e^{-d/h})^2 } 
 \,\mathrm{d}x
  \nonumber
 \\
&\quad
-
\Big\{
2(L+1)^2 \big[ L \epsilon^{(u)} +\epsilon^{(v)}+ 2 \epsilon^{(f)} \big]
+ e^{-2D/h}
\Big\}
\sum_{i=1}^M V_i 
\label{eq:Proof2Term2}
\end{align}
where the last inequality holds since for $j\neq 1$ and $x\in B_j$, we have $k(x,u),k(x,x_L)\leq e^{-D/h}$.

Note that  in~\eqref{eq:defvclusters} we defined
$v=c_{L+1} + (u-c_L)$. This gives
\[ \int_{B_{L+1}}  k^2(x,v) dx =  \int_{B_{L}}  k^2(x,u) dx\,. \]

Hence,
\begin{align*}
&\int_{x\in B_{L+1}}
\hspace{-.3in}
{(1+e^{-d/h})^2 } 
k^2(x,v)
\,\mathrm{d}x
-
\hspace{-.05in}
\int_{x\in B_L}
\hspace{-.2in}
{ \Big[ k(x,u) -  e^{-d/h} k(x,x_L) \Big]^2 }
\hspace{-.05in}
\mathrm{d}x
 \\
&=\int_{B_L} 
    \Big[(1+e^{-d/h})^2 k^2(x,u)  - 
  k^2(x,u) -  e^{-2d/h}k^2(x,x_L)
+ 2 e^{-d/h}k(x,u) k(x,x_L)\Big]
  dx
\\
&=
e^{-d/h}
\int_{B_L} 
 \Big[   (2+e^{-d/h}) k^2(x,u)  - 
  e^{-d/h}k^2(x,x_L)
 + 2 k(x,u) k(x,x_L)\Big]
  dx
  \\
&\stackrel{(a)}{\geq}
e^{-2d/h}
\int_{B_L} 
k^2(x,x_L)  \Big[1+ e^{-d/h} (2+e^{-d/h})\Big]
  dx
  \\
&\stackrel{(b)}{\geq}
 e^{-4r/h}
\int_{B_L}  
k^2(x,u)
  dx
 \,\,\stackrel{(c)}{\geq}\,\,
   e^{-6r/h}
V_L
\stackrel{(d)}{\geq} 
\frac{1}{10}V_L\,.
\end{align*}
We defined $d=\|u-x\|_p$. This implies
$ k(x,u)\geq k(x,x_L)e^{-d/h}$ which gives (a). (b) uses $d\leq 2 r$.
 For $x\in B_L$, we have $\|u-x\|_p\leq 2r$. This gives inequality (c).
The assumption  $\frac{h}{3}$ implies $r<\frac{h}{6}\,\ln 10$ which gives (d).

The assumption $D\geq 12 h \ln(2M)$ implies $D\geq 6 h \ln(2 L M)$ (since $L<M$) which gives
\[2(L+1)^2 \big[ L \epsilon^{(u)} +\epsilon^{(v)}+ 2 \epsilon^{(f)} \big]
+ e^{-2D/h} < \frac{1}{15 M}\,. \]

Plugging the above two statements in~\eqref{eq:Proof2Term2} gives the desired result.

 So for any $u\in \cup_{i=1}^{L} B_{i}$, there exists a $v\in \cup_{i=L+1}^{M} B_{i}$ which has larger score. Hence, the selection criterion based on $\score_{\data}$ would always pick a sample from a new ball to be labeled. \qedhere

 \begin{lemma}
 \label{l:inversKt}
 Let $\Lc=\{(x_i,y_i)\}_{i=1,\cdots,L}$ such that $\|x_i-x_j\|_p\geq D$ and let $u$ be such that $\|u-x_L\|=d\leq 2r$ and $\es(u)=y_L$. 
 Then there exists constants $\epsilon^{(u)} =  8 L
e^{-D/h} $ and  $\{\gamma_i\}_{i=1}^{L+1}$ satisfying
 $|\gamma_i|\leq 1$ such that 
 the interpolating function $f^u_{\es(u)}(x)$ defined
 in~\eqref{eq:defFtu} may be expressed as
 \begin{align*}
f^u_{\es(u)}(x) 
& =
 \Big[ \frac{y_L}{1+e^{-d/h}} +L\epsilon^{(u)}\gamma_{L+1} \Big]  \,k(x,u) 
 +
  \Big[ \frac{y_L}{1+e^{-d/h}} +L \epsilon^{(u)}\gamma_{L} \Big] \,k(x,x_L)
  \\
&\quad +
\sum_{i=1}^{L-1} (y_i + \epsilon^{(u)}\gamma_{i} ) K(x,x_i)\,.
\end{align*}
\end{lemma}

\begin{proof}

For fixed $u\in B_L$, define $d:=\|u-x_L\|\leq 2r$.
Step 3 in the proof of Theorem~\ref{t:clusterExplore} shows that $\es(u)=y_L$.
According to~\eqref{eq:InterpFnctu}, we have 
\begin{align*}
\widetilde{\alpha}
& = 
\Kt_u^{-1}
\begin{bmatrix}
y_1,\cdots, y_{L-1}, y_L , y_L
\end{bmatrix}^T\,.
\end{align*}
Define ${\textbf{y}}_{[1:L-1]}=[y_1,\cdots,y_{L-1}]^T$ and ${\widetilde{\alpha}}_{[1:L-1]} = [\alpha_1,\cdots,\alpha_{L-1}]$. 
To prove the statement of the lemma, we need to show that 
\begin{align}
\label{eq:Lemma2Term1}
{\widetilde{\alpha}}_{[1:L-1]} = {\textbf{y}}_{[1:L-1]} + \epsilon^{(u)}{\gamma}_{[1:L-1]}
\end{align}
 for 
 ${\gamma}_{[1:L-1]}= [\gamma_{1},\cdots,\gamma_{L-1}]^T$ 
 and 
\begin{align}
\label{eq:Lemma2Term2}
\begin{bmatrix}
\widetilde{\alpha}_L
\\ \widetilde{\alpha}_{L+1}
\end{bmatrix} = 
\frac{1}{1+e^{-d/h}}y_L 
\begin{bmatrix}
1
\\ 
1
\end{bmatrix}
+
L\epsilon^{(u)}
\begin{bmatrix}
\gamma_{L}
\\ 
\gamma_{L+1}
\end{bmatrix}\,.
\end{align}
for parameters $\gamma_i$ such that $|\gamma_i|\leq 1$. 

We will partition the matrix $\Kt_u$ defined
in~\eqref{eq:defKtYt} into the blocks corresponding to
$\{x_1,\dots,x_{L-1}\}$ and $\{x_L,x_u\}$,
\begin{align*}
\Kt_u
& = 
\begin{bmatrix}
\textbf{A} & \textbf{B}
\\
\textbf{B}^T & \textbf{D}
\end{bmatrix}
\end{align*}
where $\textbf{A} $ is a symmetric $L-1$ by $L-1$ matrix and
$\textbf{D} $ is a symmetric $2$ by $2$ matrix.  The proof essentially
follows from the fact that the elements of $\textbf{B}$ are
$k(x_i,x_L)$ and $k(x_i,x_U)$ for $i=1,\dots,L-1$, and hence very
small, so that
\begin{align*}
\Kt_u^{-1}
& \approx 
\begin{bmatrix}
\textbf{A}^{-1} & \textbf{0}
\\
\textbf{0}^T & \textbf{D}^{-1}
\end{bmatrix}
\end{align*}
To this end, first note that the diagonal elements of  $\textbf{A} $  are one. 
The off-diagonal elements of  $\textbf{A}$ are  $\textbf{A}_{i,j} = k(x_i,x_j)\leq e^{-D/h}$ for $i,j\leq L-1$.

Since $d=\|u-x_L\|$, we have  $\textbf{D}_{1,2} =\textbf{D}_{2,1}= k(x_L,u)\leq e^{-d/h}$, and 
\[
\textbf{D} = \begin{bmatrix}
 1 & e^{-d/h}
 \\
 e^{-d/h}&1
\end{bmatrix}\]

The elements of $L-1$ by $2$ matrix $\textbf{B}$ are $\textbf{B}_{i,1} = k(x_i,x_L)\leq e^{D/h}$ and $\textbf{B}_{i,2} = k(x_i,u)\leq e^{D/h}$ for $i\leq L-1$.
Since $d=\|u-x_L\|$, the application of triangle inequality gives  
$$\|u-x_i\| - \|u-x_L\| \leq  \|x_L-x_i\|\leq  \|u-x_i\| + \|u-x_L\|\,.$$
Hence, 
\begin{align}
\label{eq:lemma2Triangle}
k(u,x_i) e^{-d/h}\leq k(x_L,x_i) \leq k(u,x_i) e^{d/h}\,.
\end{align}
and consequently, $\textbf{B}_{i,1}\leq \textbf{B}_{i,2} e^{d/h}$.

Define $\textbf{F}= \textbf{B}\textbf{D}^{-1}$.
Using Schur complements, the inverse of $\Kt_u$ can be expressed as
\begin{align*}
\Kt_u^{-1}
& = 
\begin{bmatrix}
 \big( \textbf{A} - \textbf{F}  \textbf{B}^T\big)^{-1}
 & 
  - \big( \textbf{A} - \textbf{F}\textbf{B}^T\big)^{-1}
  \textbf{F}
 \\
 -\textbf{F}^T \big( \textbf{A} - \textbf{F} \textbf{B}^T\big)^{-1}
 &
 \big( \textbf{I} - \textbf{B}^T \textbf{A}^{-1} \textbf{F}\big)^{-1} \textbf{D}^{-1}
\end{bmatrix}\,.
\end{align*}
Then, we have 
\begin{align*}
{\widetilde{\alpha}}_{[1:L-1]} &= 
\big( \textbf{A} -\textbf{F} \textbf{B}^T\big)^{-1}
{\textbf{y}}_{[1:L-1]} 
  - \big( \textbf{A} - \textbf{F} \textbf{B}^T\big)^{-1}
  \textbf{F} 
  \begin{bmatrix}
  1
  \\ 1
  \end{bmatrix} y_L
  \\
  \begin{bmatrix}
\widetilde{\alpha}_L
\\ \widetilde{\alpha}_{L+1}
\end{bmatrix} 
& = 
 -\textbf{F}^T \big( \textbf{A} -\textbf{F}\textbf{B}^T\big)^{-1}
 {\textbf{y}}_{[1:L-1]} 
+
  \big( \textbf{I} - \textbf{B}^T \textbf{A}^{-1}\textbf{F}\big)^{-1} 
\textbf{D}^{-1}
\begin{bmatrix}
 1\\
  1
  \end{bmatrix} y_L
\end{align*}
Note that
\[
\textbf{D}^{-1} =
\frac{1}{1-e^{-2d/h}}
 \begin{bmatrix}
 1 & -e^{-d/h}
 \\
- e^{-d/h}&1
\end{bmatrix} \quad \quad \text{which gives} \quad\quad
\textbf{D}^{-1}
\begin{bmatrix}
 1\\
  1
  \end{bmatrix} = \frac{1}{1+e^{-d/h}}\begin{bmatrix}
 1\\
  1
  \end{bmatrix}\,.\]
Next, we will show that the elements of matrix $\textbf{F}=
\textbf{B}\textbf{D}^{-1}$ are all smaller than $e^{-D/h}$. Observe
that for $i\leq L-1$, 
\begin{align*}
\textbf{F}_{i,1} 
&= 
\frac{1}{1-e^{-2d/h}}
 \Big( k(x_i,x_L) - k(x_i,u)e^{-d/h}\Big)\,,
\\
 \textbf{F}_{i,1} 
&\stackrel{(a)}{\geq}
\frac{k(x_i,x_L) }{1-e^{-2d/h}}
 \Big( 1-1\Big)=0\,,
 \\
 \textbf{F}_{i,1} 
&\stackrel{(b)}{\leq}
\frac{k(x_i,x_L) }{1-e^{-2d/h}}
 \Big( 1-e^{-2d/h}\Big) \leq k(x_i,x_L) 
 \stackrel{(c)}{\leq}
 e^{-D/h}\,.
\end{align*}
(a) uses  $k(x_i,u)\leq e^{d/h}k(x_i,x_L)$.
(b) uses $k(x_i,u)\geq e^{-d/h}k(x_i,x_L) $. (c) uses $\|x_i-x_L\|_p\geq D$.
Similary, $\textbf{F}_{i,2}$ satisfies the same bounds.
Thus we have established that the elements of matrices $\textbf{B}$
and $\textbf{F}$ and off-diagonal elements of $\textbf{A}$ are all
smaller than $e^{-D/h}$, and recall that the diagonal elements of $\textbf{A}$ are
all one.

Next, using some algebra, the elements of matrix $\textbf{F}\textbf{B}^T$ are smaller than $2e^{-2D/h}$.
Hence,  the elements of matrix $\textbf{A} -\textbf{F} \textbf{B}^T-\textbf{I}$ have magnitude smaller than $2e^{-D/h}$ (since off-diagonal elements of matrix $\textbf{A}$ are smaller than $e^{-D/h}$).
Similar analysis as in~\eqref{eq:InverseFarK}, gives 
$\big( \textbf{A} -\textbf{F} \textbf{B}^T\big)^{-1}-\textbf{I}$  have elements smaller than $4e^{-D/h}$  (using the assumption $2L e^{-D/h}\leq 1/2$).
Also, $\textbf{F}^T\big( \textbf{A} -\textbf{F} \textbf{B}^T\big)^{-1}$
and $\big( \textbf{A} -\textbf{F} \textbf{B}^T\big)^{-1} \textbf{B}$  have elements smaller than $4Le^{-D/h}$.

Following analysis similar to~\eqref{eq:InverseFarK}, it is easy to show that
$\textbf{A}^{-1}$ has off-diagonal elements less than $2e^{-D/h}$ in
magnitude and diagonal elements satisfying $1-2e^{-D/h} \leq \textbf{A}^{-1}_{ii} \leq 1+2e^{-D/h}$.
Thus the  elements of matrix $\textbf{B}^T \textbf{A}^{-1}\textbf{F}$ are smaller than $L^2 e^{-2D/h}$.
Again, similar to the analysis in~\eqref{eq:InverseFarK}, 
$\big(\textbf{I} - \textbf{B}^T \textbf{A}^{-1}\textbf{F}\big)^{-1}-\textbf{I}$
both have elements smaller than $2 L^2e^{-2D/h}$ (using the assumption $2L^3 e^{-D/h}\leq 1/2$).

Thus we have established that the off-diagonal elements of matrix $\Kt_u^{-1}$
have magnitude smaller than $4\,L\,e^{-D/h}$ and the diagonal elements
have magnitude between $1-2Le^{-D/h}$ and $1+2Le^{-D/h}$. This fact
with the definition of $f^u_{\es(u)}(x)$ in~\eqref{eq:defFtu} gives
\begin{align*}
f^u_{\es(u)}(x) 
& =
 \big[ \frac{y_L}{1+e^{-d/h}} +L\epsilon^{(u)}\gamma_{L+1} \big]  \,k(x,u) 
 +
  \big[ \frac{y_L}{1+e^{-d/h}} +L \epsilon^{(u)}\gamma_{1} \big] \,k(x,x_L)
\\
&\quad +
\sum_{i=1}^{L-1} (y_i + \epsilon^{(u)}\gamma_{i} ) K(x,x_i) 
\end{align*}
where $\epsilon^{(u)} =  4 L^3 e^{-D/h} $.

\end{proof}